\documentclass[10pt,journal,compsoc]{IEEEtran}

\ifCLASSOPTIONcompsoc
  \usepackage[nocompress]{cite}
\else
  \usepackage{cite}
\fi
\ifCLASSINFOpdf
\else
\fi
\usepackage[utf8]{inputenc}
\usepackage{graphicx}
\usepackage{color}
\usepackage{times}
\usepackage{amsmath,mathtools}
\usepackage{amsthm}	
\usepackage{amssymb}
\usepackage{enumerate}
\usepackage{subfigure}
\usepackage{algorithm} 
\usepackage{algpseudocode}
\usepackage{booktabs}
\usepackage{tabularx}
\usepackage{array}
\usepackage{bm}
\usepackage{url}
\usepackage{bbm}
\usepackage{enumitem}
\usepackage{stackengine}
\usepackage[english]{babel}
\usepackage{xcolor}   
\usepackage{hyperref}

\pdfoutput=1

\newtheorem{theorem}{Theorem}
\newtheorem{definition}{Definition}

\DeclareMathOperator{\atantwo}{atan2}
\renewcommand{\vec}[1]{\mathbf{#1}}
\DeclarePairedDelimiter\abs{\lvert}{\rvert}

\theoremstyle{definition}
\newtheorem{proposition}{Proposition}

\newcommand{\insertimageC}[5]{ 
\begin{figure}[#5]
\centering
\includegraphics[width=#1\linewidth, clip=true]{figures/#2}
\caption{#3}
\label{#4}
\end{figure}
}

\newcommand{\insertimageStar}[5]{ 
\begin{figure*}[#5]
\centering
\includegraphics[width=#1\linewidth, clip=true]{figures/#2}
\caption{#3}
\label{#4}
\end{figure*}
}

\algnewcommand\algorithmicinput{\textbf{Input:}}
\algnewcommand\INPUT{\item[\algorithmicinput]}
\algnewcommand\algorithmicoutput{\textbf{Output:}}
\algnewcommand\OUTPUT{\item[\algorithmicinput]}

\newcommand{\comment}[1]{}
\newcommand{\rev}[1]{{#1}} 

\newcommand{\etal}{\textit{et al}.}

\newcommand{\argmax}{\operatornamewithlimits{argmax}}

\newcommand{\suchthat}{\;\ifnum\currentgrouptype=16 \middle\fi|\;}

\newcommand\SmallMatrix[1]{{%
  \normalsize\arraycolsep=0.44\arraycolsep\ensuremath{\begin{bmatrix}#1\end{bmatrix}}}}

\newcommand\MyLBrace[2]{%
\text{#2}\left\{\rule{-10pt}{#1}\right.}

\begin{document}

%
\title{Generic Primitive Detection in Point Clouds Using Novel Minimal Quadric Fits}
%
%
%
%

\author{Tolga~Birdal\,,\,
		Benjamin~Busam\,,\,
		Nassir~Navab\,,\,
        Slobodan~Ilic\,,\,
        Peter~Sturm,\,\,\,\,~\IEEEmembership{Members,~IEEE}
\IEEEcompsocitemizethanks{\IEEEcompsocthanksitem T. Birdal, B. Busam, N. Navab and S. Ilic are with the Department of Informatics, Technical University of Munich, Germany\protect\\
E-mail: tolga.birdal@tum.de, navab@cs.tum.de
\IEEEcompsocthanksitem T. Birdal and S. Ilic are with Siemens AG, Munich, Germany.\protect\\
E-mail: tolga.birdal@siemens.com, slobodan.ilic@siemens.com
\IEEEcompsocthanksitem B. Busam is with Framos GmbH, Munich, Germany.\protect\\
E-mail: b.busam@framos.com
\IEEEcompsocthanksitem P. Sturm is with INRIA, Grenoble, France\protect\\
E-mail: peter.sturm@inria.fr}
\thanks{Manuscript received June XX, 2018.}}%

%
%

\markboth{Journal of \LaTeX\ Class Files,~Vol.~XX, No.~XX, XXXX~2018}%
{Shell \MakeLowercase{\textit{et al.}}: Bare Advanced Demo of IEEEtran.cls for IEEE Computer Society Journals}
%



\IEEEtitleabstractindextext{%
\begin{abstract}
We present a novel and effective method for detecting 3D primitives in cluttered, unorganized point clouds, without axillary segmentation or type specification. We consider the quadric surfaces for encapsulating the basic building blocks of our environments - planes, spheres, ellipsoids, cones or cylinders, in a unified fashion. Moreover, quadrics allow us to model higher degree of freedom shapes, such as hyperboloids or paraboloids that could be used in non-rigid settings.

We begin by contributing two novel quadric fits targeting 3D point sets that are endowed with tangent space information. Based upon the idea of aligning the quadric gradients with the surface normals, our first formulation is exact and requires as low as four oriented points. The second fit approximates the first, and reduces the computational effort. We theoretically analyze these fits with rigor, and give algebraic and geometric arguments. Next, by re-parameterizing the solution, we devise a new local Hough voting scheme on the null-space coefficients that is combined with RANSAC, reducing the complexity from $O(N^4)$ to $O(N^3)$ (three points). To the best of our knowledge, this is the first method capable of performing a generic cross-type multi-object primitive detection in difficult scenes without segmentation. Our extensive qualitative and quantitative results show that our method is efficient and flexible, as well as being accurate.

\end{abstract}

}

\maketitle

\IEEEdisplaynontitleabstractindextext

%
\IEEEpeerreviewmaketitle

\ifCLASSOPTIONcompsoc
\IEEEraisesectionheading{\section{Introduction}\label{sec:introduction}}
\else
\section{Introduction}
\label{sec:introduction}
\fi

%
%
%
%

\IEEEPARstart{S}urface fitting and detection enjoys a rich history in computer vision and graphics communities. The problem is found particularly important because of the power of 3D surfaces to explain generic man-made structures omnipresent in every day life. Many of the constructed or manufactured objects and architecture that surrounds us are results of careful computer aided design (CAD). Some of the primary concerns of 3D computer vision, mapping and reconstruction, try to associate the visual cues acquired by various 2D / 3D sensors with those idealized CAD models, that are used in the assembly of our environments. 

One family of approaches tries to find a direct rigid association between those CAD models and 3D scenes~\cite{birdal2015point}, trying to solve the six degree of freedom (DoF) pose estimation problem. While these approaches are quite successful as the only parameters to discover are rotations and translations, they require a huge number of CAD models to generically represent the real scenarios. To overcome this limitation, inspired by the fact that all CAD models are designed using a similar set of tools, a different line of research attempts to find common \text{bases} explaining a broad set of 3D objects, and tries to detect these bases instead of individual models. Such bases that are the common building blocks of our world, and typically termed \textit{geometric primitives}. While the approaches using bases, significantly reduce the database size, usually, the bases undergo higher dimensional transformations compared to, for instance, rigid ones. Examples of the geometric primitives are splines and nurbs surfaces defined by several control points, or \textit{quadrics}, the three dimensional, nine-DoF, quadratic forms.

Thanks to their power to embody the most typical geometric primitives, such as planes, spheres, cylinders, cones or ellipsoids,
quadrics themselves were of huge interest since 80s~\cite{miller1988}. Some exemplary studies involve recovering 3D quadrics from image projections~\cite{Cross1998}, fitting them to 3D point sets~\cite{Taubin1991}, or detecting special cases of the quadratic forms~\cite{Andrews2014}. A majority of those works either put emphasis on fitting to a noisy, but isolated point set~\cite{Taubin1991,Tasdizen2001,Blane2000}, or restrict the types of shapes under consideration (thereby reduce the DoF) to devise detectors robust to clutter and occlusions~\cite{Allaire2007,Andrews2013,georgiev2016real}. 


In this work, our aim is to unite the fitting and detection worlds and present an algorithm that can simultaneously estimate all parameters of a generic nine DoF quadric, which resides in a 3D cluttered environment and is viewed potentially from a single 3D sensor, introducing occlusions and partial visibility. We craft this algorithm in three stages: (1) First, we devise a new quadric fit. Unlike its ancestors, this one uses the extra information about the tangent space to increase the number of constraints instead of regularizing the solution. This fit requires only four oriented points. We show that such construction also has a regularization effect as a by-product. (2) We then thoroughly analyze its rank properties and devise a novel null-space Hough voting mechanism to reduce the four point case to three. Three points stands out to be the minimalist case developed so far. (3) We propose a variant of RANSAC that operates on our \text{local bases}, which are randomly posited. Per each local basis, we show how to make use of the fit and the voting to hypothesize a likely quadric. Finally, we use simple clustering heuristics to group and strengthen the candidate solutions. Our algorithm works purely on 3D point cloud data and does not depend upon any acquisition modality. Moreover, it makes assumptions neither about the type of the quadric that is present in the scene nor how many are visible.  

This journal paper extends our recent CVPR publication~\cite{Birdal2018}, by providing additionally the following:
\begin{enumerate}[noitemsep]
	\item qualitative and quantitative experiments to better grasp the behavior of the proposed fit,
	\item algebraic and geometric theoretical analysis of the quadric fit
	\item improved elaborate descriptions of the method as well as accompanying pseudocode.
\end{enumerate}
\section{Related Work}
\label{sec:relatedWork}
\subsection{Quadrics} Quadrics appear in various domains of vision, graphics and robotics. They are found to be one of the best local surface approximators in estimating differential properties \cite{Petitjean2002}. Thus, point cloud normals and curvatures are oftentimes estimated with local quadrics~\cite{Zhao20163,birdal2015point}. Yan \etal  propose an iterative method for mesh segmentation by fitting local quadratic surfaces~\cite{yan2006quadric}. Yu presented a quadric-region based method for consistent point cloud segmentation~\cite{liu2016robust}. Kukelova uses quadric intersections to solve minimal problems in computer vision~\cite{kukelova2016efficient}. Uto \etal ~\cite{Uto2013} as well as Pas and Platt~\cite{Pas2013} use quadrics to localize grasp poses and in grasp planning. Quadrics have also been a significant center of attention in projective geometry and reconstruction ~\cite{Gay2017ICCV,Cross1998} to estimate algebraic properties of apparent contours. Finally, You and Zhang~\cite{You2017} used them in feature extraction from face data.



\subsection{Primitive detection} Finding primitives in point clouds has kept the vision researchers busy for a lengthy period of time. Works belonging to this category treat the primitive shapes independently~\cite{georgiev2016real}, giving rise to specific fitting algorithms for planes, spheres, cones, cylinders, etc. Planes, as the simplest forms, are the primary targets of the Hough-family~\cite{borrmann20113d}. Yet, detection of more general set of primitives made RANSAC the method of choice as shown by the prosperous Globfit~\cite{li2011globfit}: a relational local to global RANSAC algorithm. Schnabel \etal ~\cite{Schnabel2007} and Tran \etal ~\cite{Tran2015} also focused on reliable estimation using RANSAC. \rev{Monszpart et al.~\cite{monszpart2015rapter} proposed a greedy heuristic improving upon its randomized counterparts in plane detection. Oesau \etal~\cite{oesau2016planar} proposes a tandem scheme for plane detection  (by region-growing) and regularization to correct the imperfections of the hypotheses. The latter two improve upon Globfit by simultaneously extractıng the planes and the primitive relations. The plane detection has recently been lifted to structural scales~\cite{fang2018planar,czerniawski20186d}.} An interesting application of primitives is given by Qui et.al. who extract pipe runs using cylinder fitting \cite{Qiu2014}. The local Hough transform of Drost and Ilic \cite{drost2015local} showed how the detection of primitives can be made more efficient by considering the local voting spaces. Authors give sphere, cylinder and plane specific formulations targeting point clouds. Lopez \etal ~\cite{lopez2016robust} devise a robust ellipsoid fitting based on iterative non-linear optimization. Sveier \etal ~\cite{sveier2017object} suggest a conformal geometric algebra to spot planes, cylinders and spheres. Andrews' approach \cite{Andrews2014} deals with paraboloids and hyperboloids in CAD models. Even though this is slightly more generalized, paraboloids or hyperboloids are not the only geometric shapes described by quadrics.\\
Methods in this category are quite successful in shape detection, yet they handle the primitives separately. This prevents automatic type detection, or generalized modeling of surfaces.

\subsection{Quadric fitting} 
Since the 1990s generic quadric fitting is cast as a constrained optimization problem, where the solution is obtained from a Generalized Eigenvalue decomposition of a scatter matrix. Pioneering work has been done by Gabriel Taubin \cite{Taubin1991} in which a Taylor approximation to the geometric distance is made. 
This work has then been enhanced by 3L~\cite{Blane2000}, fitting a local, explicit ribbon surface composed of three-level-sets of constant Euclidean distance to the object boundary. This fit implicitly used the local surface information. Later, Tasdizen \cite{Tasdizen2001} improved the local surface properties by incorporating the surface normals as regularizers. This allows for a good and stable fit. Recently, Beale \etal ~\cite{beale2016fitting} introduced the use of a Bayesian prior to regularize the fit. All of these methods use at least nine or twelve~\cite{beale2016fitting} points. Moreover, they only use surface normals as regularizers - not as additional constraints and are also unable to deal with outliers in data. There are a few other studies~\cite{kanatani2005further,Allaire2007}, improving these standard methods, but they involve either non-linear optimization \cite{yan2012variational} or share the common drawback of requiring nine independent constraints and no outlier treatment.

\subsection{Quadric detection} Recovering general quadratic forms from cluttered and occluded scenes is a rather unexplored area. A promising direction was to represent quadrics with spline surfaces ~\cite{morwald2013geometric}, but such approaches must tackle the increased number of control points, i.e. 8 for spheres, 12 for general quadrics \cite{qin1997representing,qin1998representing}. Segmentation is one way to overcome such difficulties \cite{makhal2017grasping,morwald2013geometric}. \rev{Besl and Jain suggested a variable order segmentation-based surface fitting. They, too, use an iterative procedure where the primitive order is raised incrementally~\cite{besl1988segmentation}. This is not very different from performing individual primitive detection.} Vaskevicius et.al.~\cite{Vaskevicius2010} developed a noise-model aware quadric fitting and region-growing algorithm for segmented noisy scenes. However, segmentation, due to its nature, decouples the detection problem in two parts and introduces undesired errors especially under occlusions. Other works exploit genetic algorithms \cite{gotardo2004robust} but have the obvious drawback of inefficiency. QDEGSAC~\cite{frahm2006ransac} proposed a six-point hierarchical RANSAC, but the paper misses out an evaluation or method description for a quadric fit. Petitjean~\cite{Petitjean2002} stressed the necessity of outlier aware quadric fitting however only ends up suggesting M-estimators for future research. 

\rev{Finally, the remarkable performance of deep neural networks (DNN) for learning in 2D image domain~\cite{minto2016scene,garcia2017review} have recently been extended to 3D point clouds~\cite{ppfnet,Deng_2018_ECCV}. While 3D-PRNN~\cite{zou20173d} and PCPNet~\cite{Guerrero2018} are tailored for fitting 3D shape primitives and extracting differential surface properties respectively, their application to our problem of detecting nine-DoF primitives in real scenes containing noise, clutter and occlusions is not immediate. To the best of our knowledge, this remains to be open challenge. We would also like to stress that the community lacks comprehensive primitive detection datasets and this gives the learning algorithms a hard time to grasp all the shape variations of quadrics.}


\section{Preliminaries}
\label{sec:description}
\theoremstyle{definition}
\begin{definition}{A quadric in 3D Euclidean space is a hypersurface defined by the zero set of a polynomial of degree two}:
\begin{align}
\label{eq:quadric}
f (x,y,z)
&= Ax^2+ By^2 + Cz^2 + 2Dxy + 2Exz \\
& + 2Fyz + 2Gx + 2Hy + 2Iz + J = 0. \nonumber
\end{align}
\end{definition}
Alternatively, the vector notation $\vec{v}^T\vec{q} = 0$ is used, where:
\begin{align}
&\vec{q} = \begin{bmatrix}
           A & B & C & D & E & F & G & H & I & J
         \end{bmatrix}^T \\
&\vec{v} = \begin{bmatrix}
			x^2 & y^2 & z^2 & 2xy & 2xz & 2yz & 2x & 2y & 2z & 1
         \end{bmatrix}^T \nonumber
\end{align}
Using homogeneous coordinates, quadrics can be analyzed uniformly.
The point $\vec{x} = (x,y,z) \in \mathbb{R}^3$ lies on the quadric, if the projective algebraic equation over $\mathbb{RP}^3$ with $d_q(\vec{x}) := [\vec{x}^T  1]\vec{Q} [\vec{x}^T 1]^T=0$ holds true, where the matrix $\vec{Q} \in \mathbb{R}^{4 \times 4}$ is defined by re-arranging the coefficients:
\begin{equation}
\vec{Q} =
\SmallMatrix{
A & D & E & G\\
D & B & F & H\\
E & F & C & I\\
G & H & I & J
}
,\quad
\nabla \vec{Q} =  2
\SmallMatrix{
A & D & E & G \\
D & B & F & H \\
E & F & C & I
}.
\end{equation}
$d_q(\vec{x})$ can be viewed as an algebraic distance function. Similar to the quadric equation, the gradient at a given point can be written as $\nabla\vec{Q}(\vec{x}) := \nabla\vec{Q} [\vec{x}^T 1]^T$.
Quadrics are general implicit surfaces capable of representing cylinders, ellipsoids, cones, planes, hyperboloids, paraboloids and potentially the shapes interpolating any two of those.
All together there are 17 sub-types~\cite{Weisstein2017}.
Once $\vec{Q}$ is given, this type can be determined from an eigenvalue analysis of $\vec{Q}$ and its subspaces\rev{~\cite{Eberly2008,le2009classification}}.
Note that quadrics have constant second order derivatives and are practically smooth.
\theoremstyle{definition}
\begin{definition}{
A quadric whose matrix is of rank 2 consists of \rev{the union of two planes}: $
\vec{Q} = \mathbf{\Pi}_1 \mathbf{\Pi}_2^T + \mathbf{\Pi}_2 \mathbf{\Pi}_1^T$, 
where $\mathbf{\Pi}_1$ and $\mathbf{\Pi}_2$ are the homogeneous 4-vectors representing the two planes. A quadric whose matrix is of rank one consists of a single plane: $
\vec{Q} = \mathbf{\Pi} \; \mathbf{\Pi}^T \enspace$.}
\end{definition}
\theoremstyle{definition}
\begin{definition}{The polar plane $\mathbf{\Pi}$ of a point \rev{$\vec{p}$ with respect to a quadric} $\vec{Q}$ is $\mathbf{\Pi} = \vec{Q} \vec{p}$.
Reciprocally, $\vec{p}$ is called the pole of plane $\vec{\Pi}$.}
\end{definition}
Note that if $\vec{Q} \vec{p} = \vec{0}$, then the polar plane does not exist for $\vec{p}$;
also note that for a point that lies on the quadric, the polar plane is the tangent plane in that point.
\theoremstyle{definition}
\begin{definition}{A quadric is called central if it possesses a finite center point $\vec{c}$ that is the pole of the plane at infinity: $\vec{Q} \vec{c} \sim \begin{bmatrix} 0 & 0 & 0 & 1 \end{bmatrix}^T$ e.g. ellipsoids, hyperboloids.}
\end{definition}

\theoremstyle{definition}
\begin{definition}{A dual quadric $\mathbf{Q}^*\sim \mathbf{Q}^{-1}$ is the locus of all planes $\{\mathbf{\Pi}_i\}$ satisfying $\mathbf{\Pi}^T_i\mathbf{Q}^{-1}\mathbf{\Pi}_i=0$}.
\end{definition}
\noindent Quadric dual space is formed by the Legendre transformation, mapping points to tangent planes as covectors. \rev{Given a hypersurface, the tangent space at each point gives a family of hyperplanes, and thus defines a dual hypersurface in the dual space.} Every dual point represents a plane in the primal. Many operations such as fitting can be performed in either of the spaces~\cite{Cross1998}; or in the primal space using constraints of the dual. The latter forms a \textit{mixed} approach, \rev{involving tangency constraints}. Note that, knowing a point lies on the surface gives one constraint, and if, in addition, one knows the tangent plane at that point, then one gets two more constraints. \rev{In this paper, we will use the extra dual constraints to increase the rank of a linear system that solves for the quadric coefficients. This will in return allow us to perform fits with reduced number of points and thereby to lower the minimum number of required points. In Fig.~\ref{tab:dualspace}, we provide combinations of primal and dual constraints each of which leads to a minimal case. Note that, if we have four points, and associated tangent planes, a fit can be formulated.}

\section{Quadric Fitting to 3D Data}
\label{sec:fit}
\subsection{A new perspective to quadric fitting}
\label{sec:fit_full}
State of the art direct solvers for quadric fitting rely either solely on point sets~\cite{Taubin1991}, or use surface normals as regularizers~\cite{Tasdizen2001}.
Both approaches require at least nine points, posing a strict requirement for practical considerations, i.e. using nine points bounds the possibility for RANSAC-like fitting algorithms as the space of potential samples is $N_x^9$ where $N_x$ is the number of points. Here, we observe that typical real life point clouds make it easy to compute the surface normals (tangent space) and thus provide an additional cue. With this orientation information, we will now explain a closed form fitting requiring only four oriented points. 
\begin{figure}[t!]	
\subfigure[]{
  \setlength{\tabcolsep}{5pt}
    \begin{tabular}[b]{lcc}
          & \# Pri. & \# Dual\\
    \toprule
    PD-0 & 9 & 0 \\
    PD-1 & 7 & 1 \\
    PD-2 & 5 & 2 \\
    PD-3 & 4 & 3 \\
     \toprule
    \end{tabular}%
  \label{tab:dualspace}%
}\hfill
\subfigure[]{
\setlength{\tabcolsep}{5pt}
\begin{tabular}[b]{lccc}
          & \# Pri. & \# Dual & VS\\
    \toprule
    Plane & 1 & 1 & 0\\
    2-Planes & 2 & 2 & 0 \\
    Sphere & 2 & 2 & 1\\
    Spheroid & 2 & 2 & 3 \\
     \toprule
    \end{tabular}%
  \label{tab:votingspace}%
}
\caption{\textbf{(a)} Number of constraints for a minimal fit in Primal(P) or Dual(D) spaces. PD-i refers to $\text{i}^{\text{th}}$ combination. \textbf{(b)} Number of minimal constraints and voting space size for various quadrics.} 
\end{figure}
Similar to gradient-one fitting \cite{Tasdizen1999,Tasdizen2001}, our idea is to align the gradient vector of the quadric $\nabla \vec{Q}(\vec{x}_i)$ with the normal of the point cloud $\vec{n}_i \in \mathbb{R}^3$. However, unlike $\nabla 1$~\cite{Tasdizen1999}, we opt to use a linear constraint to increase the rank rather than regularizing the solution. This is seemingly non-trivial as the vector-vector alignment brings a non-linear constraint either of the form:
\begin{align}
\frac{\nabla \vec{Q}(\vec{x}_i)}{\|\nabla \vec{Q}(\vec{x}_i)\|} - \vec{n}_i\, = \,0\,\quad\text{or}\,\quad
\frac{\nabla \vec{Q}(\vec{x}_i)}{\|\nabla \vec{Q}(\vec{x}_i)\|} \,\cdot\, \vec{n}_i\, = \,1.
\end{align}
The non-linearity is caused by the normalization as it is hard to know the magnitude and thus the homogeneous scale in advance. We solve this issue by introducing a per normal homogeneous scale $\alpha_i$ among the unknowns and write: 
\begin{align}
\nabla \vec{Q}( \vec{x}_i) = \nabla \vec{v}^T_i \vec{q} = \alpha_i \vec{n}_i
\end{align}
Stacking this up for all $N$ points $\vec{x}_i$ and normals $\vec{n}_i$ leads to:
\begin{equation}
\label{eq:system}
\MyLBrace{14ex}{\textbf{$\vec{A}^\prime$}}\\
\begin{array}{l}
\MyLBrace{7ex}{\textbf{M}} \\
\MyLBrace{7ex}{\textbf{N}}
\end{array}
\begin{bmatrix}
  \vec{v}^T_1 & 0 & 0 & \cdots & 0 \\
  \vec{v}^T_2 & 0 & 0 & \cdots & 0 \\
  \vdots & \vdots & \vdots & \ddots & \vdots \\
  \vec{v}^T_n & 0 & 0 & \cdots & 0 \\
  {\nabla \vec{v}_1^T} & -\vec{n}_1 & \vec{0}_3 & \cdots & \vec{0}_3 \\
  {\nabla \vec{v}_2^T} & \vec{0}_3 & -\vec{n}_2 & \cdots & \vec{0}_3 \\
  \vdots & \vdots & \vdots & \ddots & \vdots \\
  {\nabla \vec{v}_n^T} & \vec{0}_3 & \vec{0}_3 & \cdots & -\vec{n}_n
\end{bmatrix}
\begin{bmatrix} A \\ B \\ \vdots \\ I \\ J \\ \alpha_1 \\ \alpha_2 \\ \vdots \\ \alpha_n \end{bmatrix}
= \vec{0}
\end{equation}
\setcounter{algorithm}{-1}
\begin{algorithm}[t!]
	\small
	\caption{Quadric \textit{full} fitting.}
	\begin{algorithmic}
		\Require{Unit normalized point set $\{\vec{x},\vec{y},\vec{z}\}$, Corresponding surface normals $\{\vec{n}_x,\vec{n}_y,\vec{n}_z\}$, A weight coefficient $\omega$}
		\Ensure{Quadric $\vec{q}=[A,B,C,2D,2E,2F,2G,2H,2I,J]^T$, Scale factors $\bm{\alpha}$}
		\State $n = \text{numel}(\vec{x})$
		\State $\vec{1} = \text{ones}(n, 1);$
		\State $\vec{0} = \text{zeros}(n, 1);$
		\State $\vec{0}_{nxn} = \text{zeros}(n, n);$
		\State $\vec{X} = [\vec{x}^2, \vec{y}^2, \vec{z}^2, \vec{x}*\vec{y}, \vec{x}*\vec{z}, \vec{y}*\vec{z}, \vec{x}, \vec{y}, \vec{z}, \vec{1}, \vec{0}_{nxn}];$
		\State $\vec{N} = [\text{diag}(\vec{n}_x); \text{diag}(\vec{n}_y); \text{diag}(\vec{n}_z)]$;
		\State $\vec{dX} = [2\vec{x}, \vec{0}, \vec{0}, \vec{y}, \vec{z}, \vec{0}, \vec{1}, \vec{0}, \vec{0}, \vec{0}; ...$
		\State $\quad\quad\quad   \vec{0}, 2\vec{y}, \vec{0}, \vec{x}, \vec{0}, \vec{z}, \vec{0}, \vec{1}, \vec{0}, \vec{0}; ...$
		\State $\quad\quad\quad  \vec{0}, \vec{0}, 2\vec{z}, \vec{0}, \vec{x}, \vec{y}, \vec{0}, \vec{0}, \vec{1}, \vec{0}];$
		\State $\vec{A} = [\vec{X}; \omega \cdot [\vec{dX}, -\vec{N}] ];$
		\State $[\sim,\sim, \vec{V}] = \text{svd}(\vec{A});$
		\State $\vec{q} = \vec{V}(1:10,n+10);$
		\State $\bm{\alpha} = \vec{V}(11:(n+10),n+10);$
		\caption{Quadric fitting, full version.}
		\label{alg:quadricfit}
	\end{algorithmic}
\end{algorithm}

\begin{figure*}[t!]
\centering
\subfigure{
\includegraphics[width=0.4525\textwidth, clip=true]{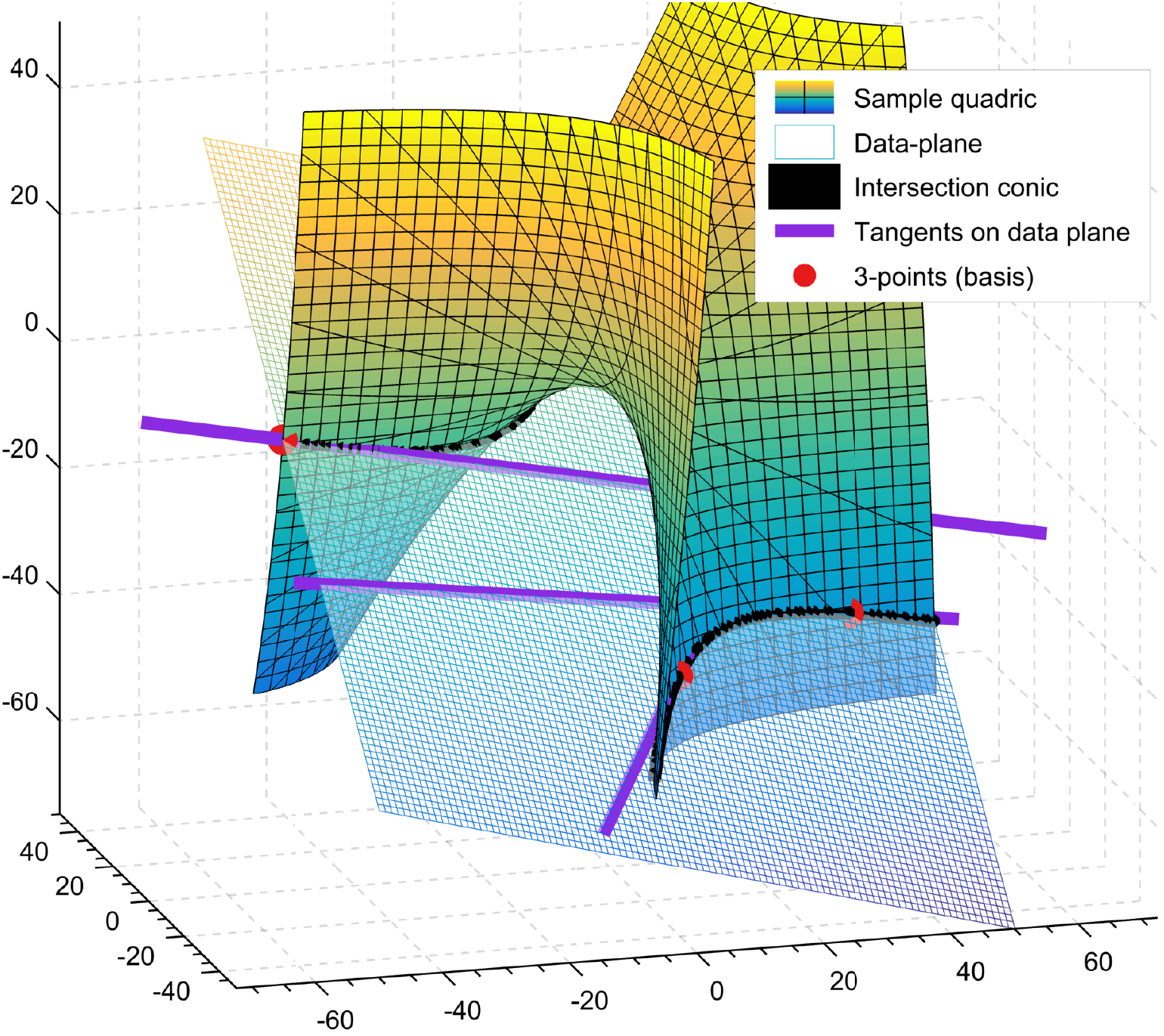}
\label{fig:dataplane}}
\hfill
\subfigure{
\includegraphics[width=0.45\textwidth, clip=true]{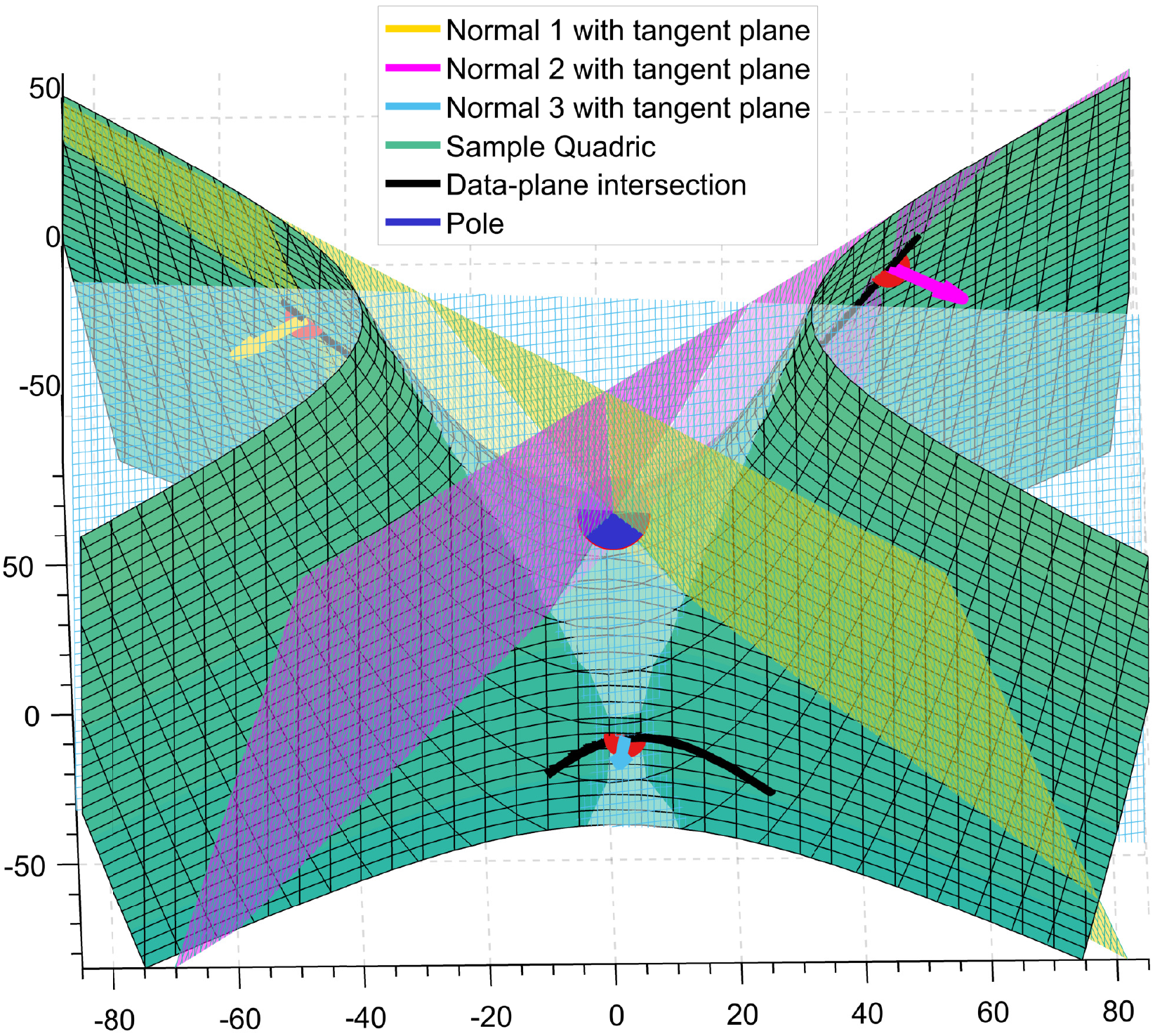}
\label{fig:tangents}}
\centering
\caption{Illustration of the geometric intuition. \textbf{(left)} Visualizations on a sample quadric: the selected basis (three-oriented points); the data-plane; the conic of intersection between data-plane and the quadric; the lines on the data-plane that are tangent to the quadric. \textbf{(right)} Exemplary drawing that shows that the tangent planes to the basis points meet at the pole (see text).}
\label{fig:saddle}
\end{figure*}
where ${\nabla \vec{v}_i^T}= {\nabla \vec{v}(\vec{x}_i)^T}$, $\vec{0}_3$ is a $3\times 1$ column vector of zeros, $\vec{A}^\prime$ is $4N\times(N+10)$ and $\bm{\alpha}=\{\alpha_i\}$ are the unknown homogeneous scales. The solution containing quadric coefficients and individual scale factors lies in the null-space of $\vec{A}^\prime$, and can be obtained accurately via Singular Value Decomposition. Alg.~\ref{alg:quadricfit} provides a MATLAB implementation of such fit. For a non-degenerate quadric, the following rank (rk) relations hold: 
\begin{align}
N = 1 \Rightarrow \text{rk}(\vec{A}) = 4,&\; N = 2 \Rightarrow \text{rk}(\vec{A}) = 7\\
N = 3 \Rightarrow \text{rk}(\vec{A}) = 9,& \text{ and }N > 3 \Rightarrow \text{rk}(\vec{A}) = 10.
\label{eq:rank}
\end{align}
We will now further investigate on this interesting behavior.
\subsection{Existence of a trivial solution for three-point case}
The problem of estimating a quadric from three points and associated normals seems initially
to be well-posed: when counting constraints and degrees of freedom, one obtains nine
on each side (each point gives one constraint, each normal two, whereas a quadric has nine degrees
of freedom). Yet, it turns out that our linear equation system always has a trivial solution besides the true one. This is summarized in Eq.~\ref{eq:rank} by providing the ranks for different cardinalities of bases. We now give further intuition and proof for this behavior:
\begin{theorem}
Three-oriented point quadric fitting, as formulated, possesses a trivial solution besides the true solution,
\rev{namely the plane spanned by the three data points. The fitting problem thus has at least a one-dimensional linear family of solutions, spanned by the true quadric and this trivial solution.}
\end{theorem}
\begin{proof}
In the following, let us call \textbf{data-plane}, the plane spanned by the three data points (coordinates only, i.e. not considering the associated normals). We illustrate this in Fig.~\ref{fig:saddle} (left).
As mentioned in $\S$~\ref{sec:fit_full} above, any rank-1 quadric consists of a single plane $\mathbf{\Pi}$ and can be written as $
\vec{Q} = \mathbf{\Pi} \; \mathbf{\Pi}^T \enspace$.
Hence, for any point $\vec{U}$ on the plane and thus on the quadric, we have $\vec{Q} \vec{U} = \vec{0}$. In our formulation of the fitting problem, this amounts to $
\vec{N} \vec{q} = \vec{0}\enspace$.
$\vec{N}$ refers to all the gradient-normal correspondence equations, stacked together (lower part of Eq.~\ref{eq:system}). We also have $
\vec{M} \vec{q} = \vec{0}.$
due to the point lying on the quadric. This means that the following vector
is a solution of the equation system: coefficients $A$ to $J$ are those of the rank-1 quadric
and the three scalars $\alpha_i$ are zero. In other words, the trivial solution is identified as the rank-1 quadric consisting of the data-plane. Hence, the estimation problem admits at least a one-dimensional linear family of solutions, spanned by the true quadric and the rank-1 quadric of the data-plane. In some cases, the dimension of the family of solutions may be higher (such as when the true
quadric is a plane).
\end{proof}
\subsubsection{A geometric explanation of the fact that the three-oriented-point problem is always under-constrained}
Despite the analytical proof, it is puzzling that nine constraints on nine unknowns are never sufficient in our problem.
Moreover, we may wonder if the existence of a trivial solution is due to our linear problem formulation or if it is generic. 
It turns out that this is generic and can be explained geometrically. To make it easier to imagine, our description will closely follow the Figures~\ref{fig:dataplane} and~\ref{fig:tangents}.
Let us decompose the estimation of the quadric in two parts, the first
part being the determination of the quadric's intersection with the data-plane.
The intersection of any quadric with any plane is in general always a conic (shown as black curve), be it real or imaginary
(the only exception is when the quadric itself contains the data-plane entirely, in which case the
``intersection'' is the entire plane).

Let us examine which constraints we have at our disposal to estimate the intersection conic.
First, the three data points lie on the conic.
Second, we know the tangent planes at the data points, to the true quadric.
Let us intersect the three tangent planes with the data plane -- the resulting three lines (shown in purple) must be tangent to our conic (the only exception occurs when one or more of these tangent
planes are identical to the data plane).

Hence, we know three points on the conic and three tangent lines -- the problem of estimating
the conic is thus in general overdetermined by one DoF.
In other words, six of the nine constraints at our disposal for estimating the quadric
are dedicated to estimating the five degrees of freedom of its intersection with the data plane.
Hence, the remaining three constraints are not sufficient to complete the estimation of the
quadric.

What are these three remaining constraints?
They refer to the orientation of the tangent planes: each of the tangent planes is defined by an
angle expressing the rotation about its intersection line with the data plane.
This angle gives one piece of information on the quadric; for three oriented points we thus have our three remaining constraints.

Note that the three tangent planes to the quadric intersect in the quadric's pole to the
data plane (see Fig.~\ref{fig:tangents}).
Hence, we can determine this pole which, as shown in appendix, lies on the line joining the centers of the possible solutions for the quadric.

Let us also note that the fact that six pieces of information (three data points and three tangent lines
to a conic in the data plane) only constrain five degrees of freedom means that these six pieces
of information are not independent from one another: in the absence of noise or other errors,
they must satisfy a consistency constraint (the fact that they define a conic).
In the presence of noise, the input information will not satisfy this constraint, meaning that
a perfect fit will not exist.
This is different in most so-called minimal estimation problems in geometric computer vision (such as
three-point pose estimation - P3P), where the computed solution is perfectly consistent with the input data.
In our case, we can expect that the computed quadric will not satisfy all constraints exactly, i.e. will not
necessarily be incident with all data points or be exactly tangent to the given tangent planes.

This gives room to different formulations for the problem, depending on how one quantifies the quality of fit. For instance, one possibility would be to impose that the quadric goes exactly through the data points, but that the tangency is only approximately fulfilled by computing the intersection conic in the data plane and minimizing some cost functions over the tangent lines.

\setcounter{algorithm}{0}
\begin{algorithm}[t!]
	\small
	\caption{Approximate quadric fitting in 10 lines of MATLAB code.}
	\begin{algorithmic}
		\Require{Unit normalized point set $\{\vec{x},\vec{y},\vec{z}\}$, Corresponding surface normals $\{\vec{n}_x,\vec{n}_y,\vec{n}_z\}$, A weight coefficient $\omega$}
		\Ensure{Quadric $\vec{q}=[A,B,C,2D,2E,2F,2G,2H,2I,J]^T$}
		\State $\vec{1} = \text{ones}(\text{numel}(\vec{x}), 1);$
		\State $\vec{0} = \text{zeros}(\text{numel}(\vec{x}), 1);$
		\State $\vec{X} = [\vec{x}^2, \vec{y}^2, \vec{z}^2, \vec{x}*\vec{y}, \vec{x}*\vec{z}, \vec{y}*\vec{z}, \vec{x}, \vec{y}, \vec{z}, \vec{1}];$
		\State $\vec{dX} = [2\vec{x}, \vec{0}, \vec{0}, \vec{y}, \vec{z}, \vec{0}, \vec{1}, \vec{0}, \vec{0}, \vec{0}; ...$
		\State $\quad\quad\quad   \vec{0}, 2\vec{y}, \vec{0}, \vec{x}, \vec{0}, \vec{z}, \vec{0}, \vec{1}, \vec{0}, \vec{0}; ...$
		\State $\quad\quad\quad  \vec{0}, \vec{0}, 2\vec{z}, \vec{0}, \vec{x}, \vec{y}, \vec{0}, \vec{0}, \vec{1}, \vec{0}];$
		\State $\vec{A} = [\vec{X}; \omega \cdot \vec{dX}];$
		\State $\vec{N} = [\vec{n}_x; \vec{n}_y; \vec{n}_z];$
		\State $\vec{b} = [\vec{0}; \omega \cdot \vec{N}];$
		\State $\vec{q} = \vec{A} \text{/} \vec{b};$
		\caption{MATLAB 10-liner for approx. quadric fitting.}
		\label{alg:quadricfitApprox}
	\end{algorithmic}
\end{algorithm}

\subsection{Regularizing with gradient norm}
\label{sec:regularfit}
Quadric fitting problem, like many others (e.g. calibration, projective reconstruction) is intrinsically of non-linear nature,
meaning that a “true” \rev{Maximum Likelihood Estimation or Maximum A Posteriori solution, minimizing a geometric distance}, cannot be achieved by a linear fit. However, our main objective in this stage is a sufficiently close and computationally efficient fit, using as few points as possible and upon which we can build our voting scheme. Despite its sparsity, for such purpose, formulation in \S~\ref{sec:fit_full} still remains suboptimal since the unknowns in Eq.~\ref{eq:system} scale linearly with $N$, leaving a large system to solve. In practice, analogous to gradient-one fitting~\cite{Tasdizen1999}, we could prefer unit-norm polynomial gradients, and thus, can write $\alpha_i=1$ or equivalently $\alpha_i \gets \bar{\alpha}$,  one common factor. This \textbf{soft constraint} will try to force zero set of the polynomial respect the local continuity of the data. \rev{Similar direction is also taken by~\cite{guennebaud2008dynamic}, for the case of spheres. However, there, authors follow a two-step fitting process, solving first the gradient and then the positional constraint, whereas we formulate a single system solving for all the shape parameters simultaneously.}
%
Such regularization also saves us from solving the sensitive homogeneous system~\cite{You2017}, and lets us re-write the system in a more compact form $\vec{A}\vec{q}=\vec{n}$:
\begin{equation}\nonumber
\vec{n}=\begin{bmatrix}
0 & 0 & \dots & {n}^1_x & {n}^1_y & {n}^1_z & {n}^2_x & {n}^2_y & {n}^2_z & \dots
\end{bmatrix}^T
\end{equation}
\begin{equation}\nonumber
\vec{q}=\begin{bmatrix}
A & B & C & D & E & F & G & H & I & J
\end{bmatrix}^T 
\end{equation}
\begin{gather}
\vec{A}=
\SmallMatrix{
x_1^2 & y_1^2 & z_1^2 & 2x_1y_1 & 2x_1z_1 & 2y_1z_1 & 2x_1 & 2y_1 & 2z_1 & 1\\
x_2^2 & y_2^2 & z_2^2 & 2x_2y_2 & 2x_2z_2 & 2y_2z_2 & 2x_2 & 2y_2 & 2z_2 & 1\\
&&&& \vdots &&&&&\\
2x_1 & 0 & 0 & 2y_1 & 2z_1 & 0 & 2 & 0 & 0 & 0\\
0 & 2y_1 & 0 & 2x_1 & 0 & 2z_1 & 0 & 2 & 0 & 0\\
0 & 0 & 2z_1 & 0 & 2x_1 & 2y_1 & 0 & 0 & 2 & 0\\
2x_2 & 0 & 0 & 2y_2 & 2z_2 & 0 & 2 & 0 & 0 & 0\\
0 & 2y_2 & 0 & 2x_2 & 0 & 2z_2 & 0 & 2 & 0 & 0\\
0 & 0 & 2z_2 & 0 & 2x_2 & 2y_2 & 0 & 0 & 2 & 0\\
&&&& \vdots &&&&&
}
\label{eq:systemApprox}
\end{gather} 
$\vec{A}$, only $4N \times 10$, is similar to the $\vec{A}^\prime$ in §~\ref{sec:fit_full} and gets full rank for four or more oriented points. In fact, it is not hard to show that the equations in rows are linearly dependent, which is why we get diminishing returns when we add further constraints.
Note that by removing the scale factors from the solution, we also solve the sign ambiguity problem, i.e. the solution to Eq.~\ref{eq:system} can result in negated gradient vectors. 
To balance the contribution of normal induced constraints we introduce a scalar weight $w$, leading to the ten-liner MATLAB implementation as provided in Alg.~\ref{alg:quadricfitApprox}. 



In certain cases, to obtain a type-specific fit, a minor redesign of $\vec{A}$ tailored to the desired primitive suffices (see §.~\ref{sec:exptype}).
If outliers corrupt the point set, a four-point RANSAC could be used. However, below, we present a more efficient way to calculate a solution to Eq.~\ref{eq:systemApprox} rather than using a naive RANSAC on four-tuples by analyzing its solution space. The next section can also be used as a generic method to solve any fitting problem formulated as a linear system, more efficiently.
\section{Quadric Detection in Point Clouds}
\label{sec:detection}
We now factor in clutter and occlusions into our formulation and explain a new pipeline to detect quadrics in 3D data.
\theoremstyle{definition}
\begin{definition}{A basis $\vec{b}$ is a subset composed of a fixed number of scene points ($b$) and hypothesized to lie on the sought surface.}
\end{definition}
Our algorithm operates by iteratively selecting bases from an input scene. Once a basis is fixed, an under-determined quadric fit parameterizes the solution and attached to this basis, a local accumulator space is formed. All other points in the scene are then paired with this basis to vote for the potential primitive. To discover the optimal basis, we perform RANSAC, iteratively hypothesizing different basis candidates and voting locally for probable shapes. Subsequent to such joint RANSAC and voting, we verify resulting hypotheses with efficient two-stage clustering and score functions such that multiple quadrics can be detected without repeated executions of the algorithm. We will now describe, in detail, the voting and the bases selection, respectively.

\subsection{Parameterizing the solution space}
\label{sec:parametrize}
Linear system in Eq.~\ref{eq:system} describes an outlier-free closed form fit. To treat the clutter in the scene, a direct RANSAC on nine-DoF quadric appears to be trivial. Yet, it has two drawbacks:
1) evaluating the error function many times is challenging, as it involves a scene-to-quadric overlap calculation in a geometric meaningful way.
2) even with the proposed fitting, selecting random four-tuples from the scene might be slow in practice. 

An alternative to RANSAC is Hough voting. However, $\vec{q}$ has nine DoFs and is not discretization friendly. The complexity and size of this parameter space makes it hard to construct a voting space.
Instead, we will now devise a local search. For this, let $\vec{q}$ be a solution to the linear system in (\ref{eq:systemApprox}) and $\vec{p}$ be a particular solution. $\vec{q}$ can be expressed by a linear combination of homogeneous solutions $\bm{\mu}_i$ as:
\begin{align}
\label{eq:nullspace}
&\vec{q}= \vec{p} + \sum_i^{d_n} \lambda_i \bm{\mu}_i \\
&= \vec{p} +
\begin{bmatrix}
\bm{\mu}_1 & \bm{\mu}_2 & \cdots
\end{bmatrix}
\begin{bmatrix}
\lambda_1 & \lambda_2 & \cdots
\end{bmatrix}^T = \vec{p} + \vec{N}_A\bm{\lambda}.\nonumber
\end{align}
The dimensionality ${d_n}$ of the null space $\vec{N}_A$  depends on the rank of $\vec{A}$, which is directly influenced by the number of points used: ${d_n}=10-rk(\vec{A})$. The exact solution could always be computed by including more points from the scene and validating them, i.e. by a local search. For that reason, the fitting can be split into distinct parts: first a parametric solution is computed, such as in Eq. \ref{eq:nullspace}, using a subset of points $\vec{b}=\{\vec{x}_1,...,\vec{x}_m\}$ which lie on a quadric. We refer to subset $\vec{b}$ as the \textit{basis}. Next, the coefficients $\bm{\lambda}$, and thus the solution, can be obtained by searching for other point(s) $(\vec{x}_{m+1},...,\vec{x}_{m+k})$ which lie on the same surface as $\vec{b}$. 

\begin{proposition}
\label{thm:fastLambda}
If two point sets $\vec{b}=(\vec{x}_1,...,\vec{x}_m)$ and $\vec{X}=(\vec{x}_{m+1},...,\vec{x}_{m+k})$ lie on the same quadric with parameters $\vec{q}$, then the coefficients $\bm{\lambda}=
\begin{bmatrix}
\lambda_1 & \lambda_2 & \cdots
\end{bmatrix}^T$
of the solution space (\ref{eq:nullspace}) are given by the solution of the system:
\begin{equation}
\label{eq:lemma1}
(\vec{A}_k  \vec{N}_A) \bm{\lambda} = \vec{n}_k - \vec{A}_k \vec{p}
\end{equation}
where $\vec{A}_k$, $\vec{n}_k$ are the linear constraints of the latter set $\vec{b}'$ in form of (\ref{eq:systemApprox}), $\vec{p}$ is a particular solution and $\vec{N}_A$ is a stacked null-space basis as in (\ref{eq:nullspace}), obtained from $\vec{b}$.
\end{proposition}
\begin{proof}\let\qed\relax
Let $\vec{q}$ be a quadric solution for the point set $(\vec{x}_1,...,\vec{x}_m)$ and let $(\vec{A}_k,\vec{n}_k)$ represent the $4 k$ quadric constraints for the $k$ points $\vec{X} = (\vec{x}_{m+1},...,\vec{x}_{m+k})$ in form of (\ref{eq:system}) with the same parameters $\vec{q}$.
As $\vec{x}_i \in \vec{X}$ by definition lies on the same quadric $\vec{q}$, it also satisfies $\vec{A}_k \vec{q} = \vec{n}_k$.
Inserting Eq. \ref{eq:nullspace} into this, we get:
\begin{align}
\label{eq:lambdas}
\vec{A}_k (\vec{p}+\vec{N}_A\bm{\lambda}) &= \vec{n}_k\\
(\vec{A}_k  \vec{N}_A) \bm{\lambda} &= \vec{n}_k - \vec{A}_k \vec{p}
\label{eq:lambdas2}
\end{align}
\end{proof}
Solving Eq. \ref{eq:lambdas2} for $\bm{\lambda}$ requires a multiplication of a $4 k \times 10$ matrix with a $10 \times m$ one and ultimately solving a system of $4 k$ equations in $m$ unknowns. Once $\vec{N}_A$ and $\vec{p}$ are precomputed, it is much more efficient to evaluate Eq.~\ref{eq:lemma1} for $k < m$ rather than re-solving the system (\ref{eq:systemApprox}). This resembles updating the solution online for a stream of points. For our case, the amount of streamed points will depend on the size of the basis, as explained below. 
\subsection{Local voting for quadric detection}
\label{sec:voting}
Given a fixed basis composed of $b$ points $(b>0)$ as in Fig. \ref{fig:basis}, a parametric solution can be described. The actual solution can then be found quickly by using Prop.~\ref{thm:fastLambda} by incorporating new points lying on the same quadric as the basis. Thus, the problem of quadric detection is de-coupled into
1) finding a proper basis and
2) searching for compatible scene points. In this section, we assume the basis is correctly found and explain the search by voting. For a fixed basis $\vec{b}_i$ on a quadric, we form the null-space decomposition of the under-determined system $\vec{A}_i\vec{q}=\vec{n}_i$. We then sample further points from the scene and compute the required coefficients $\bm{\lambda}$. Thanks to Prop.~\ref{thm:fastLambda}, this can be done efficiently. Sample points lying on the same quadric as the basis (inliers) generate the same $\bm{\lambda}$ whereas outliers will produce different values. Therefore we propose to construct a voting space on $\bm{\lambda}$ attached to basis $\vec{b}_i$ and cast votes to maximize the consensus, only up to the locality of the basis. Fig.~\ref{fig:basis} illustrates this configuration. The size of the voting space is a design choice and depends on the size of the basis $\vec{b}_i$ vs. the DoFs desired to be recovered (see Fig. \ref{tab:votingspace}). 

While many choices for the basis cardinality are possible (and the formulation in \S~\ref{sec:parametrize} allows for all), we find from \rev{Fig.~\ref{tab:dualspace} that using a three-point basis is advantageous for a generic quadric fit - having three dual points, reduces the minimum number of required primal (incidence) constraints to only four. 
And by the rank analysis given in Eq.~\ref{eq:rank}, we see that it is possible to trade one point off to 1D local search as opposed to two-point vs 3D search for the five-point case.
}
\insertimageC{1}{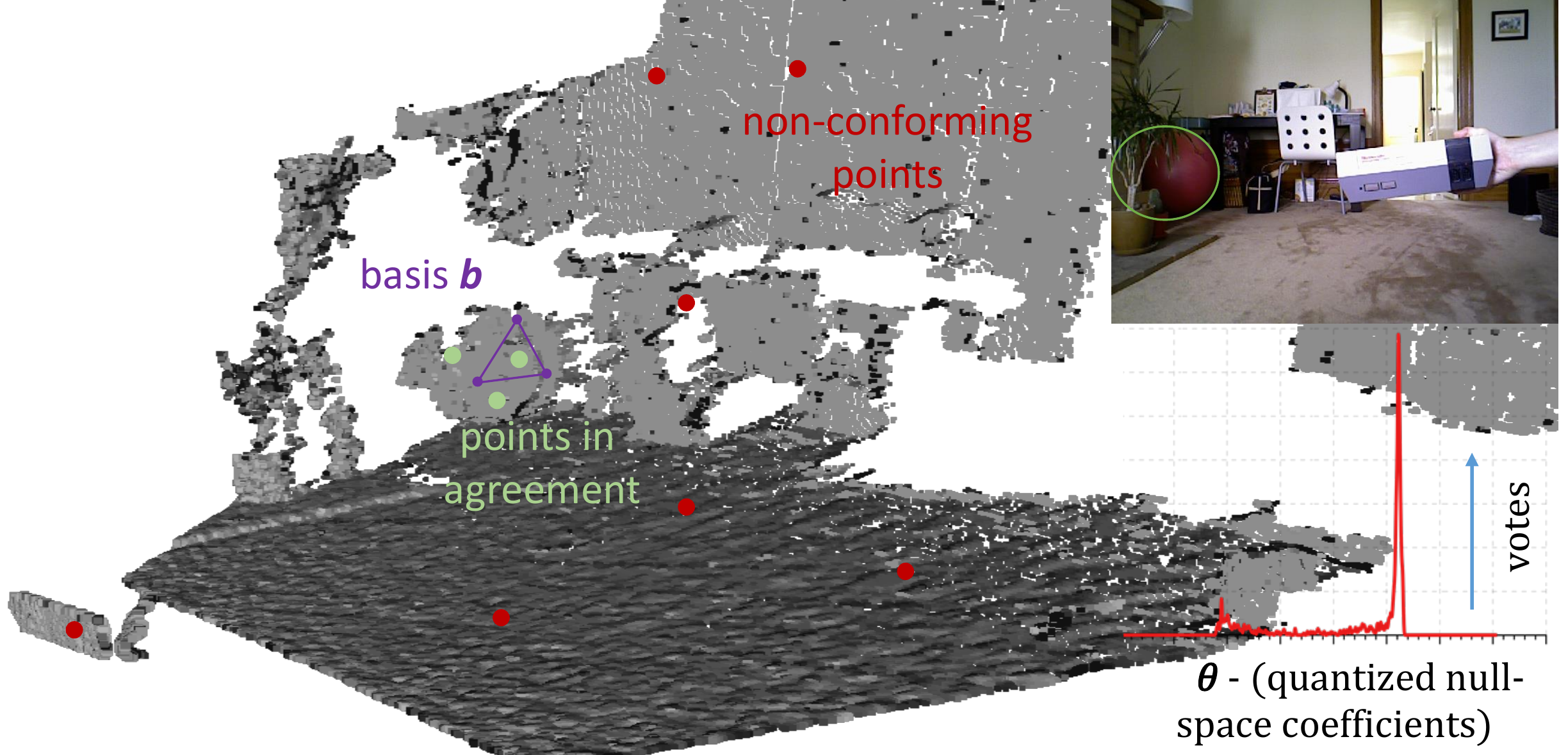}{Once a basis is randomly hypothesized, we look for the points on the same surface by casting votes on the null-space. The sought pilates ball (likely quadric) is marked on the image and below that lies the corresponding filled accumulator by KDE~\cite{rosenblatt1956remarks}.}{fig:basis}{t!}
\begin{figure*}[t!]
\centering
\includegraphics[width=\linewidth]{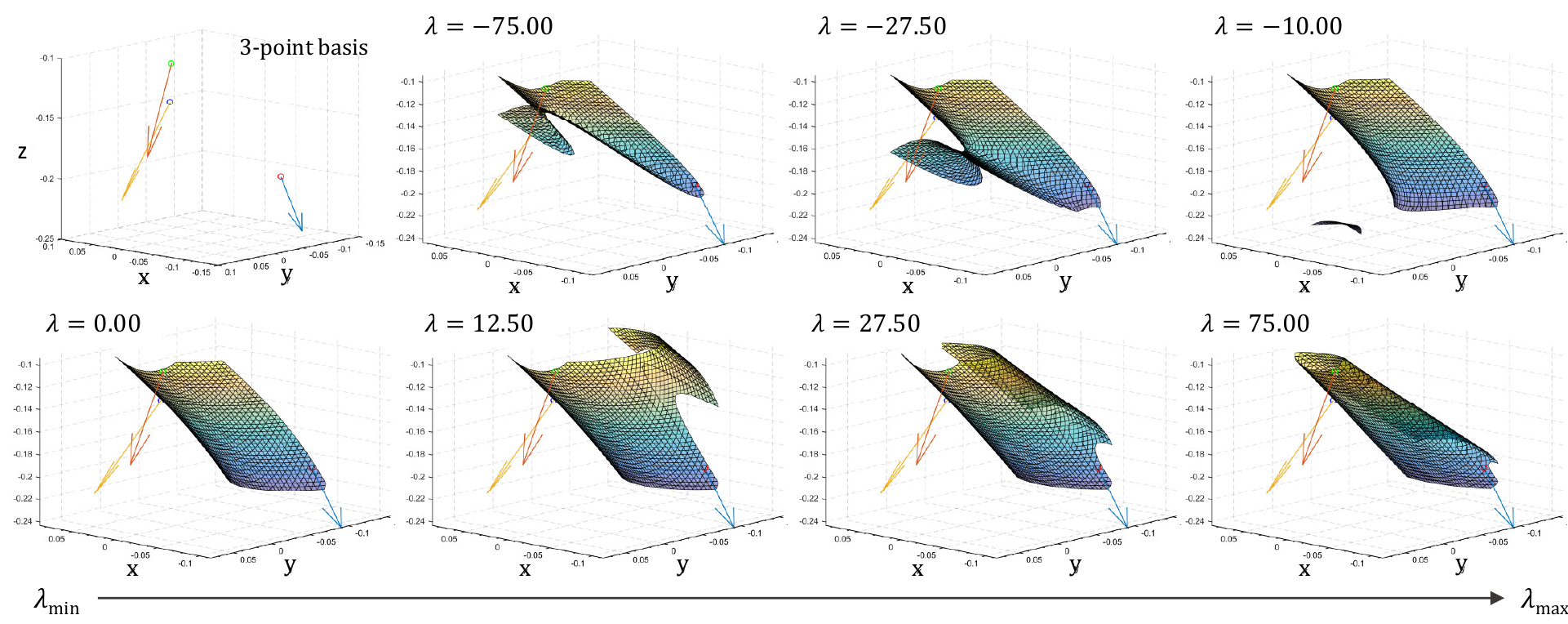}
\label{fig:lambdas}
   \caption{Effect of $\lambda$ on the surface geometry. We compute null-space decomposition for a fixed basis and vary $\lambda$ from -75 to 75 to generate different solutions $\vec{q}$ along the line in the solution space. The plot presents the transition of the surface controlled by $\lambda$.\vspace{-7pt}}
\end{figure*}
\subsection{Efficient computation of voting parameters for a 1D voting space}
Adding a fourth sample point $\vec{x}_4$ completes $rk(\vec{A}) = 10$ and a unique solution can be computed, as described above.
Yet, as we will select multiple $\vec{x}_4$ candidates per basis, hypothesized in a RANSAC loop, an efficient scheme is required, i.e. once again, it is undesirable to re-solve the system in Eq. \ref{eq:systemApprox} for each incoming $\vec{x}_4$ tied to the basis.
It turns out that the solution can be obtained directly from Eq. \ref{eq:nullspace}:

\begin{proposition}
\label{thm:lambda_one}
If the null-space is one dimensional (with only 1 unknown) it holds $\bm{\lambda}\vec{N}_A = \lambda_1 \bm{\mu}_1$ and the computation in Prop.~\ref{thm:fastLambda} reduces to the explicit form:
\begin{align}
\label{eq:lambdas_fast}
\lambda_1 = \frac{\vec{A}_1\vec{N}_A}{\|\vec{A}_1\vec{N}_A\|^2} \cdot (\vec{n}_1 - \vec{A}_1\vec{p})
\end{align}
\end{proposition}
\begin{proof}
Let us re-write Eq. \ref{eq:lambdas2} in terms of the null space vectors: $\lambda_1(\vec{A}_1\bm{\mu}_1) = \vec{n}_1 - \vec{A}_1 \vec{p}$.
A solution $\lambda_1$ can be obtained via Moore-Penrose pseudoinverse as $\lambda_1 = (\vec{A}_1\bm{\mu}_1)^{+}(\vec{n}_1 - \vec{A}_1 \vec{p})$.
Because for one-dimensional null spaces, $\vec{A}_1\bm{\mu}_1$ is a vector ($\bm{v}$), for which the ${}^{+}$ operator is defined as: $\bm{v}^{+}=\bm{v} / (\bm{v}^T\bm{v})$. Substituting this in Eq. \ref{eq:lemma1} gives Eq. \ref{eq:lambdas_fast}.
\end{proof}
Prop.~\ref{thm:lambda_one} enables a very quick computation of the parameter hypothesis in the case of an additional single oriented point. A MATLAB implementation takes ca. $30 \mu s$ per $\lambda$. \rev{Note that for the minimal system we propose, four incidence (primal) and three tangent plane alignment (dual) constraints are sufficient. This means that the normal of the fourth sample point does not contribute to the set of constraints for a \textit{minimal} fit. Hence, we use this piece of information for the verification of the fit. We only accept to vote a candidate quadric if the gradient of the fitted surface agrees with the surface normal of the fourth point:
\begin{equation}
\frac{\nabla \vec{Q}(\vec{x}_4)}{\| \nabla \vec{Q}(\vec{x}_4) \|} \cdot \vec{n}(\vec{x}_4) > \tau_n.
\end{equation}
We typically set $\tau_n\approx 0.85$ in order to tolerate certain noise.}
\subsection{Quantizing \texorpdfstring{$\bm{\lambda}$}{asdf} for voting}
Unfortunately, $\bm{\lambda}$ is not quantization-friendly, as it is unbounded and has a non-linear effect on the quadric shape (Fig. \ref{fig:lambdas}).
Thus, we seek to find a geometrically meaningful transformation to a bounded and well behaving space so that quantization would lead to little bias and artifacts. From a geometric perspective, each column of $\vec{N}_A$ in Eq. \ref{eq:nullspace} is multiplied by the same coefficient $\vec{\lambda}$, corresponding to the slope of a high dimensional line in the solution space. Thus, it could as well be viewed as a rotation. For 1D null-space, we set: 
\begin{equation}
\theta = \text{atan2}\Big(\,\frac{y_2-y_1}{x_2-x_1}\,\Big)
\end{equation}
where $[x_1,y_1,\cdots]^T=\vec{p}$ and $[x_2,y_2 \cdots ]^T$ is obtained by moving in the direction $\vec{N}_A$ from the particular solution $\vec{p}$ by an offset $\lambda$.\footnote{Simple $tan^{-1}(\lambda)$ could work but would be more limited in the range.} This new angle $\theta$ is bounded and thus easy to vote for. As the null-space dimension grows, $\vec{\lambda}$ starts to represent hyperplanes, still preserving the geometric meaning, i.e. for $d>1$, different $\bm{\theta}=\{\theta_i\}$ can be found. 

\rev{Even though $\theta$ behaves better than $\lambda$ for voting, we still can not guarantee a unimodal distribution such that a single peak can be identified unambiguously. Nevertheless, thanks to the local voting, the case that one distribution is noisy or \textit{misty} will be handled when other random bases are selected. It is more likely that the peaks coming from different bases are concentrated around the same mode, rather than a single peak of one accumulator. Besides, we have empirically observed that in many real cases, even when the distribution is amodal, a single peak is prominent when the sampled fourth is in a reasonable vicinity of the basis. 
}





\setcounter{algorithm}{1}
\begin{algorithm}[t!]
	\small
	\caption{Combined 3-point RANSAC \& local voting for Quadric discovery.}
	\label{alg:ransac}
	\begin{algorithmic}
		\Require{Unit normalized point set $\vec{P}$, Corresponding surface normals $\vec{N}$, A weight coefficient $\omega$, Minimum vote threshold $s_{min}$}
		\Ensure{Quadrics $\vec{Q}=\{\vec{q}_i\}$}
		\State $(\vec{S},\vec{N}) \gets$ Sample scene $(\vec{P},\vec{N})$.
		\While{ $\text{!satisfied }$}\Comment{seek the best global candidates}
			\State $\vec{b}_i \gets \text{Pick a random 3 point-basis from }(\vec{S},\vec{N})$.
			\State $(\vec{A},\vec{n}) \gets \text{Form system in eq. (4) using } \vec{b}_i$
			\State $(\vec{p},\bm{\mu})\gets \text{Perform null space decomposition - eq.7.}$
			\State $\vec{V} \gets \text{Initial voting space of length $\#$ bins}$
			\State $\bm{\Lambda} \gets \{\}$
			\For{$\text{all }\vec{p}_i \text{ in }\vec{P}$}\Comment{local voting}
			\State Compute $\lambda_i$ by including $\vec{p}_i$ using eq. 11.
			\If {\big($1-\frac{\nabla\vec{Q}(\vec{p}_i)}{\|\nabla\vec{Q}(\vec{p}_i)\|} \cdot \vec{n}_i <\tau$\big)} \Comment goodness of fit
			\State Quantize: $\theta \gets tan^{-1}({\lambda_i})$ (using $\text{atan2}$).
			\State $\vec{V}[{\theta}]++$ \Comment{accumulate}
			\State $\bm{\Lambda}[\theta] \gets  \bm{\Lambda}[\theta] \cup \lambda_i $
			\EndIf
			\EndFor
		\State $\theta^* \gets \argmax_{j}\vec{V}_j$\Comment{best candidate in quantized space}
		\If {${\big|\bm{\Lambda}[\theta^*]\big|}>s_{min}$}
		\State $\lambda_{\text{best}} \gets {\sum_k \bm{\Lambda}[\theta^*][k]}\, / \, {\big|\,\bm{\Lambda}[\theta^*]\,\big|}$\Comment{best local coefficient}
		\State $\vec{q} \gets \vec{p}+\lambda_{\text{best}}\bm{\mu}$\Comment{best local solution}
		\State $\vec{Q} \gets \{\vec{Q}, \vec{q}\}$
		\EndIf
		\EndWhile
		\State $\vec{Q} \gets$ mean of the clusters in $\vec{Q}$ using distance $\vec{d}_{\text{close}}$
		\State $\vec{Q} \gets$ mean of the clusters in $\vec{Q}$ using distance $\vec{d}_{\text{far}}$
		\State $\vec{Q} \gets $ sort(score($\vec{Q}$))
		\caption{Combined RANSAC \& Local Voting.}
	\end{algorithmic}
\end{algorithm}

\subsection{Hypotheses aggregation} Up until now, we have described how to find plausible quadrics given local triplet bases. As mentioned, to discover the basis lying on the surface, we employ RANSAC~\cite{Fischler1981}, where each triplet might generate a hypothesis to be verified. Many of those will be similar as well as dissimilar. 
Thus, the final stage of the algorithm aggregates the potential detections to reduce the number of candidate surfaces and to increase the per quadric confidence.
Not to sacrifice further speed, we run an agglomerative clustering similar to~\cite{birdal2015point} in a coarse to fine manner: 
First a fine (\textit{close}) but fast distance measure helps to cluster the obvious hypotheses. 
Second, a coarse (\textit{far}) one is executed on these cluster centers.

\begin{definition}{Our distance computation is two-fold: Whenever two quadrics are close, we approximate their distances as in Eq. \ref{eq:qdistance} ($d_{close}$),
where $\vec{I} \in \mathbb{R}^{4 \times 4}$ is the identity matrix and $\mathbbm{1}: \mathbb{R} \rightarrow \left\{0, 1\right\}$ the indicator function.
We use the pseudoinverse just to handle singular configurations.
If the shapes are far, such manifold-distance becomes erroneous and we use a globally consistent metric.
To do so, we define a more geometric-meaningful distance using the points on the scene itself ($d_{far}$):
\begin{align}
\nonumber
&d_{close}(\vec{Q}_1, \vec{Q}_2) 
:= \mathbbm{1}( \|\vec{q}_1-\vec{q}_2 \|_1 < \tau)
\cdot \| \vec{Q}_1 \vec{Q}_2^{+} - \vec{I} \|_F \\
\label{eq:qdistance}
&d_{far}(\vec{Q}_1, \vec{Q}_2):= \\
&1-\frac{1}{K}\sum\limits_{i=1}^{K}
\mathbbm{1}\left(\lvert\vec{x}_i^T\vec{Q}_1\vec{x}_i\rvert<\tau\right) \cdot
\mathbbm{1}\left(\lvert\vec{x}_i^T\vec{Q}_2\vec{x}_i\rvert<\tau\right) \cdot \nonumber\\
& \mathbbm{1}\left(1-\vec{n}_i \cdot \nabla\vec{Q}_1(\vec{x}_i) <\tau_n\right) \cdot
\mathbbm{1}\left(1-\vec{n}_i \cdot \nabla\vec{Q}_2(\vec{x}_i) <\tau_n\right). \nonumber
\end{align}
$\{\vec{x}_i\}$ denote the $K$ scene samples.}
\end{definition}
Note that, algebraic but efficient $d_{close}$ lacks geometric meaning, while slower $d_{far}$ can, to a certain extent, explain the geometry.
Finally, the quadrics are sorted w.r.t. their scores, evaluated pseudo-geometrically by point and normal-gradient compatibility according to Def.~\ref{dec:score}:
\insertimageStar{1}{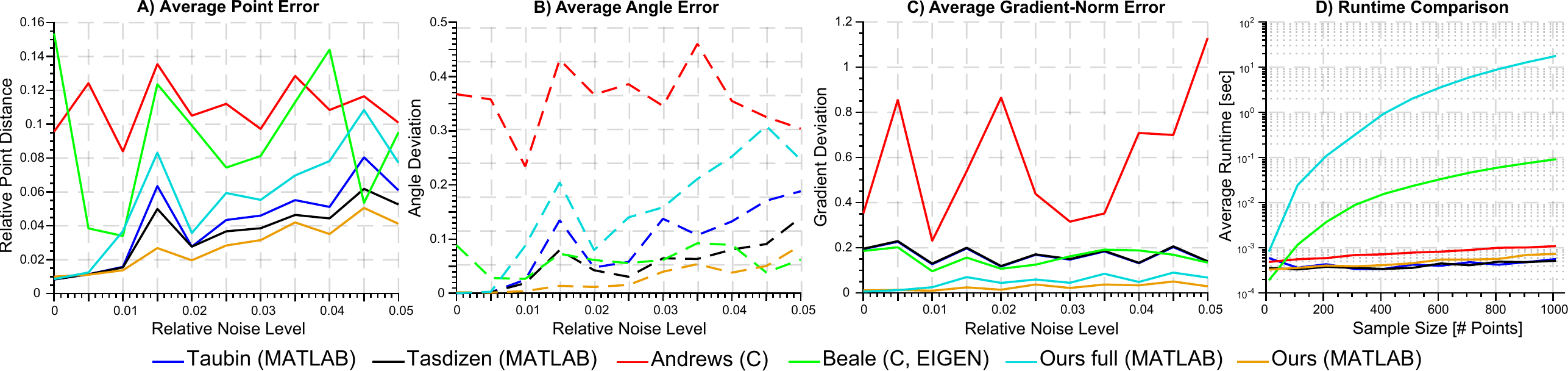}{Synthetic evaluations. The plot depicts mean geometric errors on points \textbf{(a)} and mean angular errors \textbf{(b)} for different quadric fitting methods.
The per point error is measured as the average point-to-mesh distance from every ground truth vertex to the fitted quadric.
The angular error (dashed) is computed as the negated dot product between quadric gradient and the ground truth normal. Moreover, \textbf{(c)} shows the average error of the gradient norm compared to the ground truth and
\textbf{(d)} gives speed and detection rate on synthetic data.\vspace{-7pt}}{fig:synth_exp}{t!}
\begin{definition}\label{dec:score}
The score of a quadric is defined to be:
\begin{align}
S_{\vec{Q}, \vec{X}} = 
\frac{1}{K}\sum\limits_{i=1}^K
\mathbbm{1}\left(\lvert\vec{x}_i^T\vec{Q}\vec{x}_i\rvert<\tau\right)
\mathbbm{1}\left(1-\vec{n}_i \cdot \nabla\vec{Q}(\vec{x}_i)<\tau_n\right) \nonumber
\end{align}
\end{definition}
\rev{While other distance metrics, such as spectral decompositions are possible, we found these to be sufficient in our experiments. The final algorithm is summarized in Alg. 3.}


\begin{figure}[t!]
\begin{center}
 \subfigbottomskip =-1cm
\includegraphics[width=0.48\columnwidth]{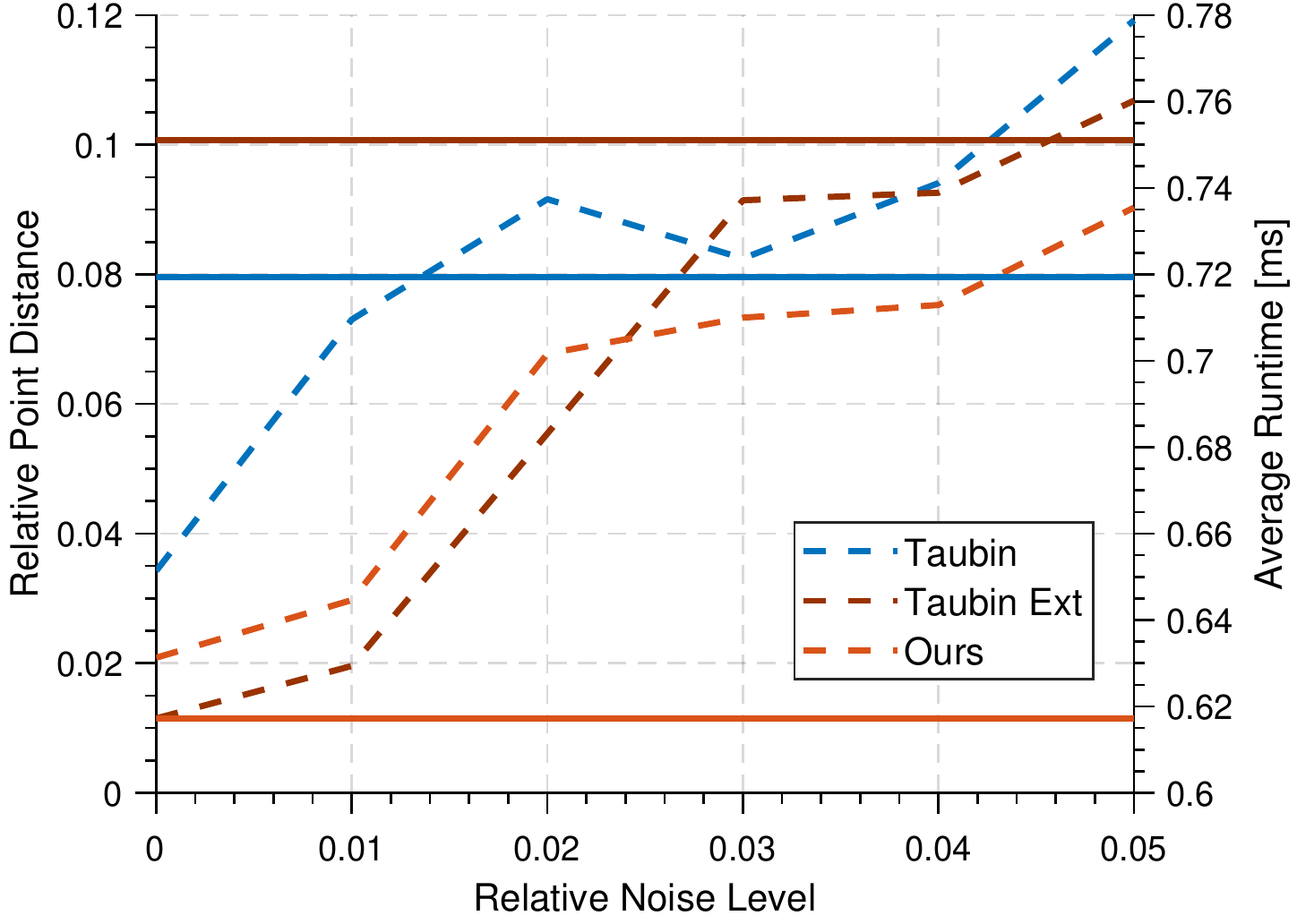}
\hfill
\includegraphics[width=0.22\textwidth]{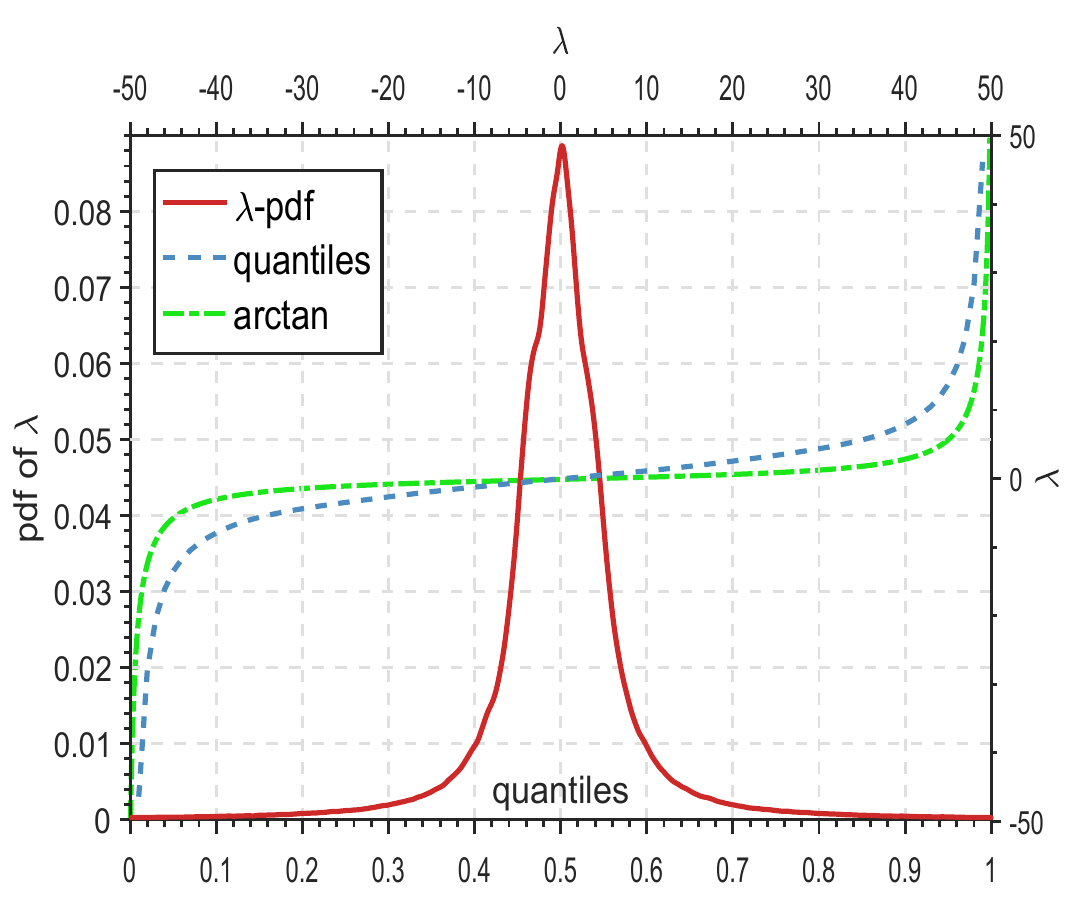}
\end{center}
   \caption{\textbf{a} (left). Effect of extended point neighborhood to the fitting.  \textbf{b} (right). Statistical distribution of the solution-space coefficient and our quantization function: PDF (red curve) and inverse CDF (dashed blue-curve) of $\lambda$ over collected data, and $tan^{-1}$ function (green-line). Note that our quantization function is capable of explaining the empirical data.}
\label{fig:lambda_plot_pdf}
\end{figure}
\section{Experimental Evaluation and Discussions}
\label{sec:experiments}
\subsection{Implementation details}
Prior to operation, we normalize the point coordinates to lie in a unit ball to increase the numerical stability~\cite{hartley1997defense}. Next, we downsample the scene using a spatial voxel-grid enforcing a minimum distance of $\tau_s \cdot diam(\vec{X})$ between the samples ($\tau_s=0.03$)~\cite{birdal2017iros}. The required surface normals are computed by the local plane fitting~\cite{Hoppe1992}. As planes are singular quadrics and occupy large spaces of 3D scenes, we remove them. To do so, we convert our algorithm to a type specific plane detector, which happens to be a similar algorithm to~\cite{drost2015local}. 
Next, influenced by the smoothness of quadrics, we use Difference of Normals (DoN)~\cite{Ioannou2012} to prune the points not located on smooth regions. 
What follows is an iterative selection of triplets to conduct the three-point RANSAC:
We first randomly draw the initial point of the basis $\vec{x}_1$. Once $\vec{x}_1$ is fixed, we query the points in a large enough vicinity, whose normals differ enough to form the three-point basis $\vec{b}$. 
The rest of the points are then randomly selected respecting these criteria.
To avoid degenerate configurations, we skip the basis if it does not result in a rank-9 matrix $\vec{A}$. In addition, to reducing the bias towards repeating bases, we hash the seen triplets and avoid duplicates. 

\subsection{Synthetic tests of fitting and ablation studies}
To asses the accuracy of the proposed fitting, we generate a synthetic test set of multitudes of random quadrics and compare our method with the fitting procedures of Taubin~\cite{Taubin1991}, Tasdizen~\cite{Tasdizen2001}, Andrews~\cite{Andrews2014}, and Beale~\cite{beale2016fitting}. We propose two variants: \textbf{Ours full} will refer to Alg. 1, whereas \textbf{Ours} is the regularized one (Alg. 2).

\subsubsection{Quantitative assessments}
\label{sec:quantitative}
Prior to run, we add Gaussian noise to the ground-truth vertices with $\sigma=[0\% - 5\%]$ relative to the size $s$ of the quadric.
At each noise level, ten random quadrics are tested. We perform not single but twenty fits per set. For the constrained fitting method~\cite{Andrews2014} we pre-specified the type, which might not be possible in a real application. 
We then record and report the average point-to-mesh distance and the angle deviation as well as the runtime performances in Fig. \ref{fig:synth_exp}.
Although, our fit is designed to use a minimal number of points, it also proves robust when more points are added and is among the top fitters for the distance and angle errors. In addition, Fig.~\ref{fig:synth_exp}c shows that the errors on the gradient magnitudes obtained by our quadrics. We achieve the least errors, showing that gradient norms align well with the ground truth, favoring the validity of our approximation/regularization. 
Next, looking at the noise assessments, we see that our full method performs the best on low noise levels but quickly destabilizes. This is because the system might be biased to compute correct norms rather than the solution and it has increased parameters.
We believe the reason for our compact fit to work well is the soft constraint where the common scale factor acts as a weighted regularizer towards special quadrics. When this constraint cannot be satisfied, the solution settles for a very acceptable shape.

In a further test, we include the six neighboring points of each of seven query points to perform a standard Taubin-fit. We call this \textit{Taubin-42}. Fig.~\ref{fig:lambda_plot_pdf}a shows that  while the error of our method is on par with \textit{Taubin-42}, we are more robust at higher noise values and more efficient with a runtime advantage of ca. $22\%$.

Since for for a visually appealing fit, the normal alignment is crucial, we next present a qualitative evaluation.
\subsubsection{Qualitative assessments}
We synthesized a random saddle quadric and performed a random point sampling over its surface. Next, we added Gaussian noise on the sample points and computed the normals. To resolve the sign ambiguity, each normal is flipped in the direction of ground truth gradient. We plot the results of the fitting in Fig.~\ref{fig:tests_noise}.
Even in presence of little noise only some methods fail to estimate the correct geometry, mostly due to the bias towards certain shape~\cite{Andrews2014,beale2016fitting}.
Our approach is able to recover the correct surface even in presence of a severe noise. Also the effect of our regularization is visible on the last column, which possesses the best visual quality.
\insertimageStar{1}{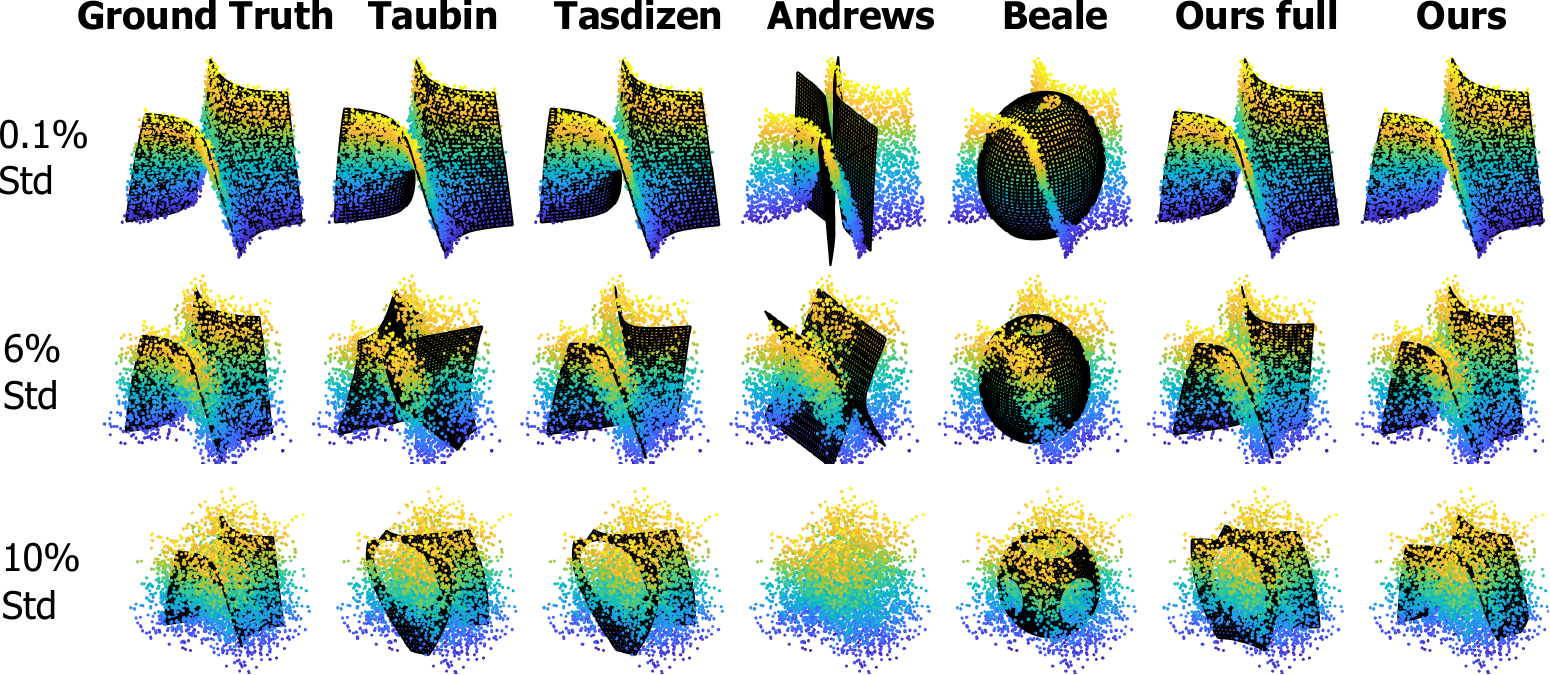}{Synthetic tests at various noise levels for different fitting methods. Gaussian noise is added to the point coordinates as well as the estimated normal. The standard deviation varies from 0.11\% of the visible quadric size to 10\%.}{fig:tests_noise}{t!}
\begin{figure}[t!]
\centering
\subfigtopskip =-1cm
 \subfigbottomskip =-1cm
\includegraphics[width=0.44\columnwidth]{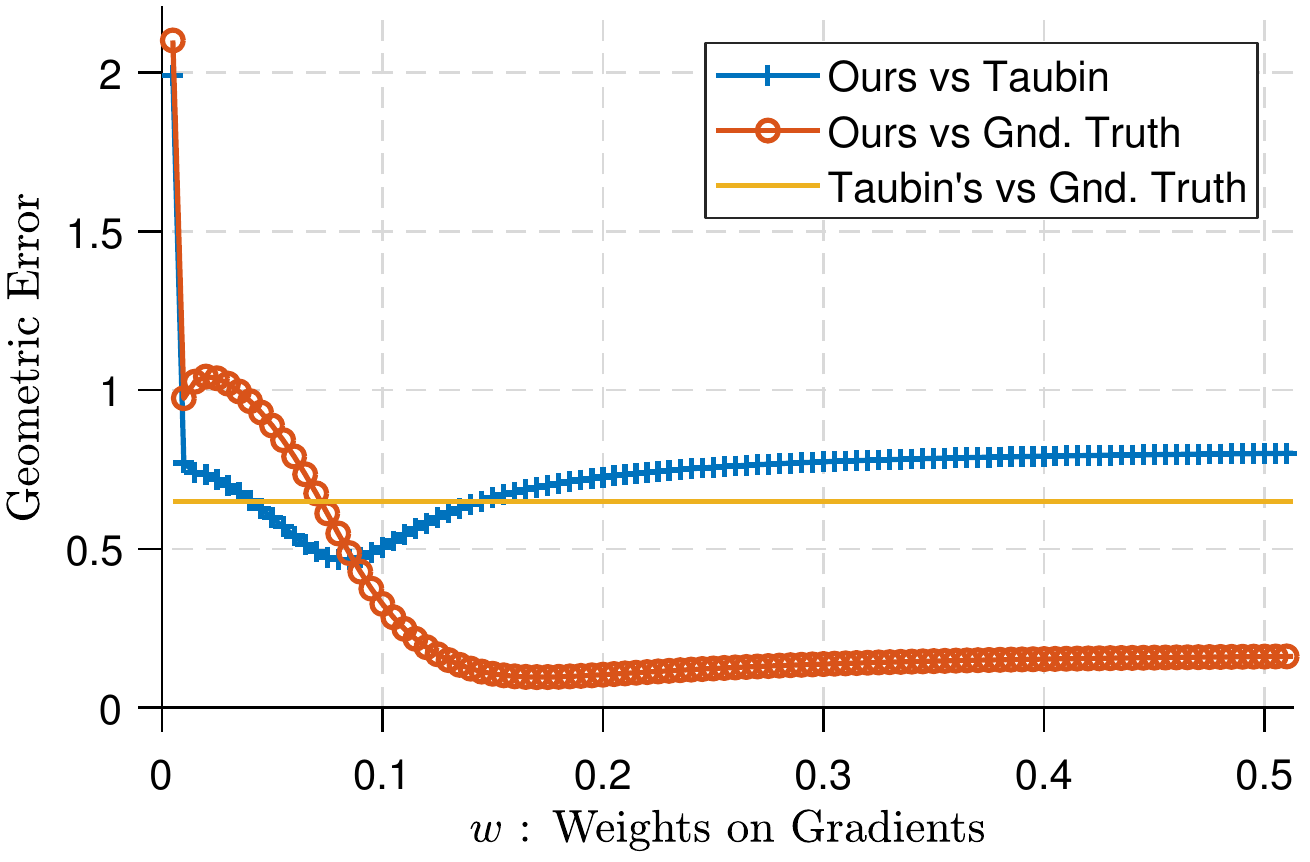}
\hfill
\includegraphics[width=0.265\textwidth]{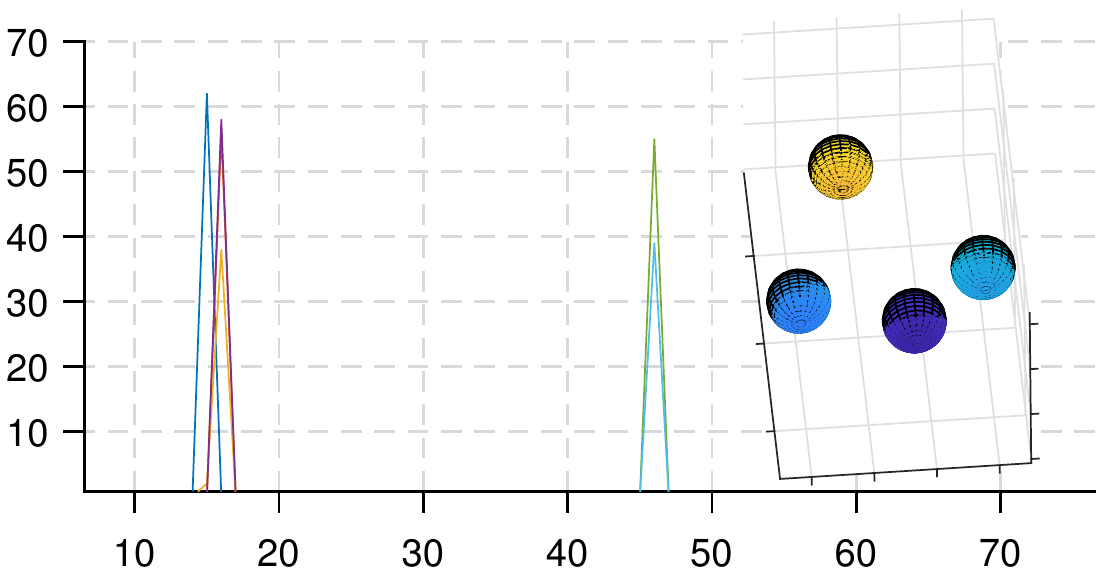}
\caption{\rev{\textbf{a} (left). Effect of the weight $w$ on the quality of the fit (Alg. 2). \textbf{
b} (right). Voting spaces due to different bases selection ($\theta$ vs $\#$ votes).}}
\label{fig:ablation}
\end{figure}

It is of interest to see whether our regularized fit can estimate correct surface normals as well as direction. Thus, a second test was performed to qualitatively observe the gradient properties in more detail. For this, a series of randomly generated quadrics is fitted by Taubin's and our method and the gradients are analyzed both in terms of magnitude and phase, as shown in Fig.~\ref{fig:normal_eval}. 
\insertimageStar{1}{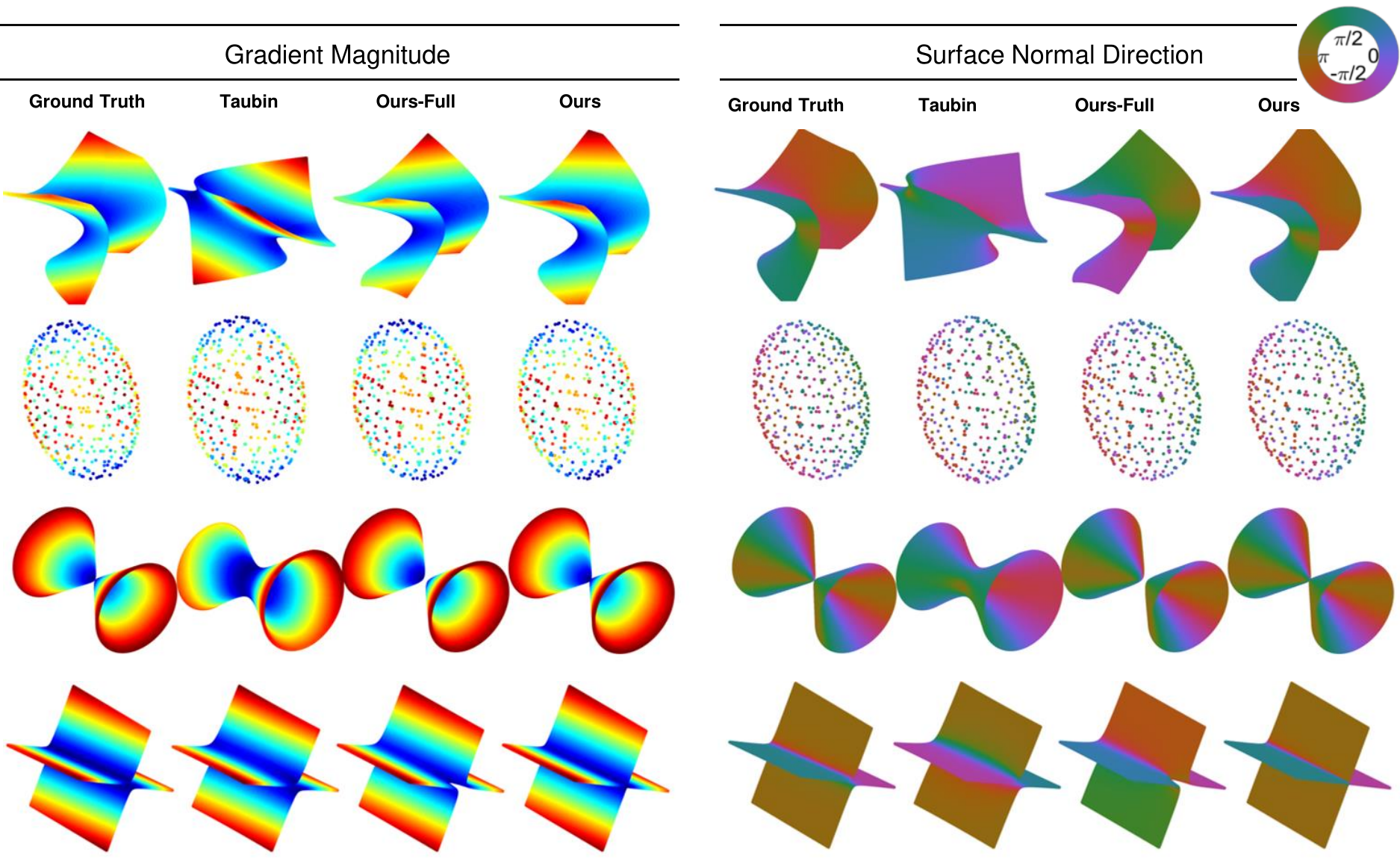}{Qualitative evaluation of surface normals. Randomly generated quadrics are used as ground truth and fitting is performed. The estimation results of gradient magnitude and angle (phase) is color coded on the surface. For color selection, we use a jet-like temperature map for the gradient magnitude, where blue denotes the lowest and red denotes the highest magnitudes. For the phase, an angular map as show in the color-bar is used. The ideal case is given in ground-truth against which the methods compete.\vspace{-3pt}}{fig:normal_eval}{t!}
Due to our explicit treatment of the gradients, it can be clearly seen that the gradient direction is recovered better. Moreover, the right side of Fig.~\ref{fig:normal_eval} also shows that our approximate approach yields the expected results, while the full method could sometimes generate inconsistent gradient signs, as the scale factors are estimated individually. 
Finally, it is qualitatively visible in Fig.~\ref{fig:normal_eval} that the magnitudes recovered by our method are compatible to the ground truth. Such improvement without sacrificing gradient quality validates the regularizing nature of our approach.


\subsubsection{Is \texorpdfstring{${\atantwo}$}{asdf} a valid transformation for \texorpdfstring{$\bm{\lambda}$}?}
To assess the practical validity of the quantization, we collect a set of 2.5 million oriented point triplets from several scenes and use them as bases to form the underdetermined system $\vec{A}$. We then sample the fourth point from those scenes, compute $\lambda$ and establish the probability distribution $p(\lambda)$ for the whole collection to calculate the quantiles, mapping $\lambda$ to bins via the inverse CDF. A similar procedure has been applied to cross ratios in~\cite{Birdal2016}. We plot the findings together with the $\atantwo$ function in Fig. \ref{fig:lambda_plot_pdf}b and show that the empirical distribution and $\atantwo$ follow similar trends, justifying that our quantizer is well suited to the data.
\rev{\subsubsection{Effect of weighting on the fit}
We now investigate the effect of weighting parameter $w$ on the fit. For a selection of eight noisy points, located on three different synthetic quadrics, we vary $w$ and plot, in Fig.~\ref{fig:ablation}a, the geometric errors attained by Alg. 2, against the ground truth and Taubin fit. While too low of $w$ hurts our fit, there is a large range of values $w\in[0.08,1.0]$, where we can outperform~\cite{Taubin1991}. 
\subsubsection{How do voting spaces look like?}
To provide insights on the local voting spaces of the angles $\theta$, we sample different random bases on four synthetic quadrics as embedded in Fig.~\ref{fig:ablation}b, and collect the votes along with the quantized bins. These accumulators are shown in the same figure, each with a different color. It is observed that, the voting spaces are myst-free and a only single mode emerges, thanks to the maximum distance threshold selected between the basis and the paired point. It is still possible to obtain multiple modes if the threshold is unrealistically picked. The consensus votes correspond to the true shape, and erroneous votes spread randomly. 
}
\begin{figure}[t!]
\centering
\includegraphics[width=\linewidth, clip=true]{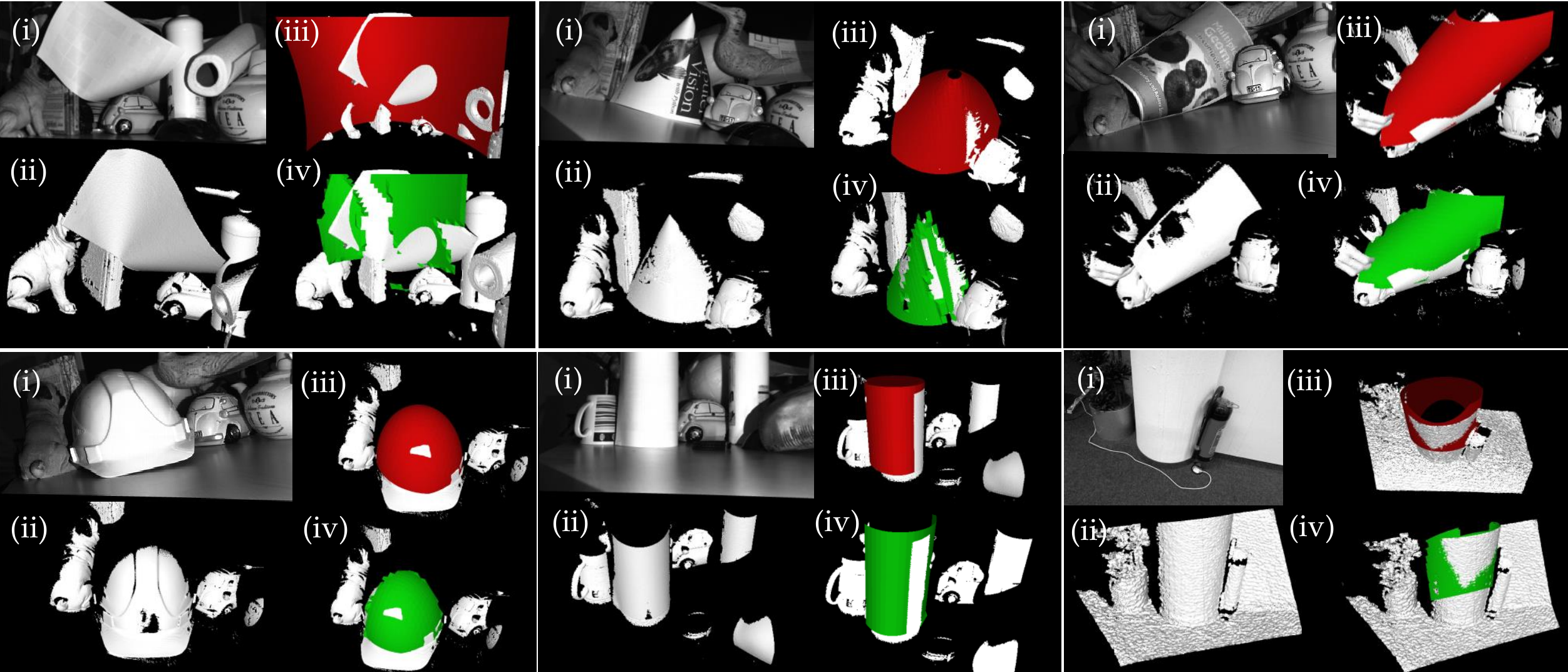}
\caption{\textbf{(i)} Images captured by an industrial structured light sensor and Kinect (last image). \textbf{(ii)} 3D scene. \textbf{(iii)} Detected quadric, shown without clipping. \textbf{(iv)} Quadric in (iii) clipped to the points it lies on.}
\label{fig:qualitative}
\end{figure}
\subsection{Real experiments on quadric detection}
Besides synthetic tests where self evaluation was possible, we assess the quality of generic primitive detection, on 3 real datasets:
\begin{enumerate}
\item{\textbf{Our Dataset}} First, because there are no broadly accepted datasets on quadric detection, we opt to collect our own. To do so, we use an accurate phase-shift stereo structured light scanner and capture 35 3D scenes of 5 different objects within clutter and occlusions.
Our objects are three bending papers, helmet, paper towel and cylindrical spray bottle.
Other objects are included to create clutter.
To obtain the ground truth, for each scene, we generated a visually acceptable set of quadrics using 1) \cite{Schnabel2007} when shapes represent known primitives 2) by segmenting the cloud manually and performing a fit, when the quadric type is not available.
Each scene then contains 1-3 ground truth quadrics. This dataset has low noise, but a high amount of clutter and partial visibility due to the FOV limitations of the sensor.
\item{\textbf{Large Objects}} Kinect sensor is widely accepted in computer vision community. Thus, it is desirable to see the performance of our generic and type-specific fit approaches on the Kinect depth images. To this end, we adapt the large objects RGB-D scan dataset of~\cite{Choi2016}. From this dataset, we sample only the scenes containing objects, that could roughly be explained by geometric primitives. These scenes include apples, globes, footballs, or other small balls. Tab.~\ref{tab:accuracy} summarizes the objects used. Example detections are also shown in Fig.~\ref{fig:spherefit}. We also augment this dataset with a Pilates Ball sequence that we collect. This sequence involves a lot of partial visibility, clutter and fast motions (see appendix).
\item{\textbf{Cylinders}} Finally, we use a subset of the ITODD dataset~\cite{Drost2017}, designed to evaluate object pose estimators. Our subset, \textit{Cylidners}, includes 14 scenes of varying number of cylinders, from one to ten, as shown in Fig.~\ref{fig:cylinders}a. Again, we use RGB images only to ease the visual perception. 
\end{enumerate}

\begin{figure}[t!]
\subfigure[Speed and detection results.]{
\includegraphics[width=0.48815\columnwidth, clip=true]{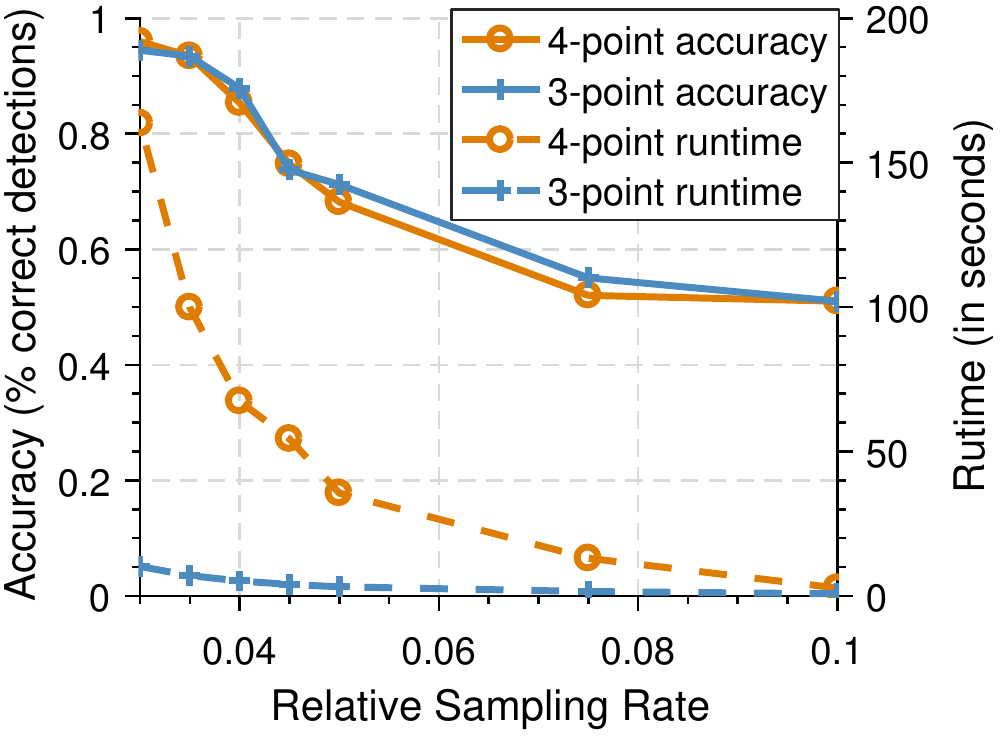}
\label{fig:real}}
\hfill
\subfigure[Distribution of errors of sphere fit.]{
\includegraphics[width=0.46\columnwidth, clip=true]{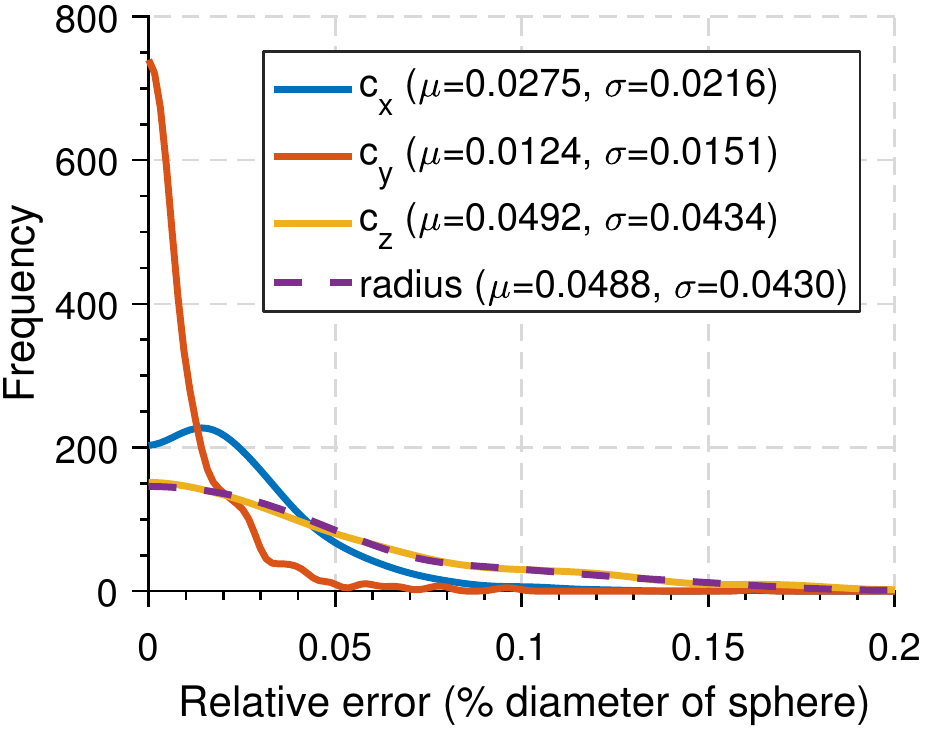}
\label{fig:poseacc}}
\caption{Experiments on real datasets. Our three-point variant clearly outperforms the four-point algorithm in terms of speed, while not sacrificing quality. For three-point method, the mean errors indicate that the localization accuracy is on the levels of the sampling rate $\tau\approx0.05$, agreeing to the success of our algorithm.}
\end{figure}
\subsubsection{Evaluations on detection accuracy} 
To assess the detection accuracy, we manually count the number of detected quadrics aligning with the ground truth in \textit{Our Dataset}.
We compared the four-point and three-point algorithms, both of which we propose.
We also tried the naive nine-point RANSAC algorithm (with \cite{Taubin1991}), but found it to be infeasible when the initial hypotheses of the inlier set is not available. Fig. \ref{fig:qualitative} visualizes the detected quadrics both on our dataset and on the 3D data captured by Kinect .
Fig. \ref{fig:real} presents our accuracy over different sampling rates and the runtime performance.
Our three-point method is on par with the four-point variant in terms of detection accuracy, while being significantly faster.
Next, we also evaluate our detector on the large objects dataset of \cite{Choi2016} without further tuning. Tab.~\ref{tab:accuracy} shows $100\%$ accuracy in locating a frontally appearing ellipsoidal rugby ball over a $1337$ frame sequence without type prior. While such scenes are not particularly difficult, it is noteworthy that we manage to generate the similar quadric repeatedly at each frame within $5\%$ of the quadric diameter.
\insertimageStar{1}{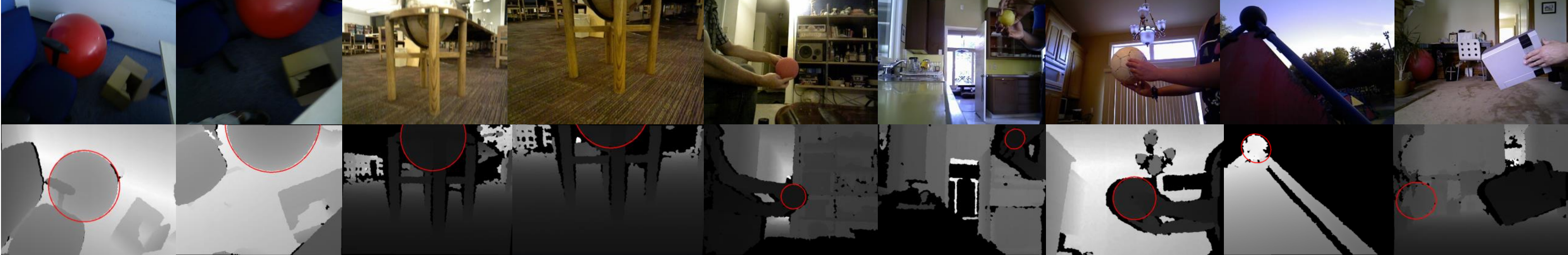}{Qualitative visualizations of sphere detection in the wild: Our algorithm successfully detects primitives in difficult scenarios including clutter, occlusions and large distances. Note that the sphere is detected in 3D only using the point clouds of depth images and we draw the apparent contour of the quadric. The RGB pictures are also included in the top row to ease the visual perception.}{fig:spherefit}{t!}
\subsubsection{How fast is it?}
As our speed is influenced by the factors of closed form fitting, RANSAC and local voting, we evaluate the fit and detection separately. Fig.~\ref{fig:synth_exp}d shows the runtime of fitting part. Our method scales linearly due to the solution of an $4N \times 10$ system, but it is the fastest approach when $<300$ points are used. Thus, it is more preferred for a minimal fit. Fig. \ref{fig:real} then presents the order of magnitude speed gain, when our four-point C++ version is replaced by three-points without accuracy loss. Although the final runtime is in the range of 1-2 seconds, our three-point algorithm is still the fastest known method in segmentation free detection.
\subsubsection{How accurate is the fit?}
To evaluate the pose accuracy on real objects, we use closed geometric objects of known size from the aforementioned datasets and report the distribution of the errors, and their statistics. We choose \textit{football} and \textit{pilates ball 1} as it is easy to know their geometric properties (center and radius). We compare the radius to the true value while the center is compared to the one estimated from a non-linear refinement of the sphere. Our results are depicted in Fig. \ref{fig:poseacc}. Note that the errors successfully remain about the used sampling rates ($\tau\approx0.05$), which is as best as we could get.
\begin{table}[t!]
  \centering
  \setlength\tabcolsep{2 pt}
  \resizebox{\columnwidth}{!}{
    \begin{tabular}{lccccc}
    \toprule
          & Dataset & \# Objects & \multicolumn{1}{c}{Type} & Occlusion & \multicolumn{1}{c}{Accuracy} \\
    \midrule
    Pilates Ball 1 & Ours  & 580   & \multicolumn{1}{c}{Generic} & Yes   & \multicolumn{1}{c}{94.40\%} \\
    Rugby Ball & \cite{Choi2016}    & 1337  & \multicolumn{1}{c}{Generic} & No    & \multicolumn{1}{c}{100.00\%} \\
    Pilates Ball 2 & \cite{Choi2016}    & 1412  & \multicolumn{1}{c}{Sphere} & Yes   & \multicolumn{1}{c}{100.00\%} \\
    Big Globe & \cite{Choi2016}    & 2612  & \multicolumn{1}{c}{Sphere} & Yes   & \multicolumn{1}{c}{90.70\%} \\
    Small Globe & \cite{Choi2016}    & 379   & \multicolumn{1}{c}{Sphere} & Yes   & \multicolumn{1}{c}{56.90\%} \\
    Apple & \cite{Choi2016}    & 577   & \multicolumn{1}{c}{Sphere} & Yes   & \multicolumn{1}{c}{99.60\%} \\
    Football & \cite{Choi2016}    & 1145  & \multicolumn{1}{c}{Sphere} & Yes   & \multicolumn{1}{c}{100.00\%} \\
    Orange Ball & \cite{Choi2016}    & 270   & \multicolumn{1}{c}{Sphere} & Yes   & \multicolumn{1}{c}{93.30\%} \\
    \bottomrule
    \end{tabular}%
    }
  \caption{Detection accuracy on real datasets.}
  \label{tab:accuracy}%
\end{table}%
\subsubsection{Type-specific detection} 
\label{sec:exptype}
It is remarkably easy to convert our algorithm to a type specific one by re-designing matrix $\vec{A}$. Here, we propose a sphere-specific detector. Let us write any sphere in the following matrix form:
\begin{equation}
\label{eq:sphere}
\mathbf{Q}_{s} = \begin{bmatrix}
\vec{I}_3 & -\vec{c}\\
-\vec{c}^T & \lVert \vec{c} \rVert-r^2
\end{bmatrix}
\end{equation}
where $\vec{c} = ( c_x, c_y, c_z )$ and $r$ are the geometric parameters (center and radius) of the sphere. Rotation does not affect spheres and our $\vec{A}\vec{q}=\vec{b}$ formulation in \S~\ref{sec:regularfit} then simplifies to:
\begin{gather}
\vec{A}=
\begin{bmatrix}
\| \vec{x}_1 \|^2 & 2x_1 & 2y_1 & 2z_1 & 1\\
\| \vec{x}_2 \|^2 & 2x_2 & 2y_2 & 2z_2 & 1\\
&& \vdots &&\\
2\vec{x}_1 & & \vec{I}_3 & & \vec{0}_3 \\
2\vec{x}_2 & & \vec{I}_3 & & \vec{0}_3 \\
&& \vdots &&
\end{bmatrix}, 
\vec{b}=\begin{bmatrix}
0 \\ 0 \\ \vdots \\ \vec{n}_1 \\ \vec{n}_2 \\ \vdots
\end{bmatrix} , 
\vec{q}=\begin{bmatrix}
A \\ B \\ C \\ D \\ E
\end{bmatrix}
\end{gather}
Due to the geometric interpretability, at the scoring phases, we can use the point-to-sphere distance as:
\begin{equation}
d_{\text{point $\rightarrow$ sphere}}(\mathbf{p}, \mathbf{q}) = \abs{( \lVert \vec{p}-\vec{c} \rVert_2 - r )}
\end{equation}
where $\vec{p}$ is the point to compute the distance. A sphere to sphere distance (used in clustering) can be obtained by:
\begin{equation}
d_{\text{sphere $\rightarrow$ sphere}}(\mathbf{q}_1, \mathbf{q}_2) = \frac{1}{2} (\abs{r_1-r_2} + \lVert \vec{c}_1-\vec{c}_2 \rVert_2)
\end{equation}
Center and radius of the sphere can always be obtained from the quadric form $\mathbf{Q}_{s}$ as described in eq. \ref{eq:sphere}. 

Note that $rk(\vec{A})=4$ if one point is available, leaving only one free parameter which forms a single dimensional null-space. Geometrically, this means that the radius cannot be resolved from a single point. Yet, by fixing another point, one can vote locally as explained in \S~\ref{sec:voting}. While at this stage Drost and Ilic \cite{drost2015local} prefer to vote for radius explicitly, we vote for the null-space coefficient. The difference is that \cite{drost2015local} involves trigonometric computations before the voting stage, but vote linearly for the geometric parameter, whereas we keep linearity until the voting stage but vote for the non-linear angle $\theta$ corresponding to $\lambda$. Our approach evaluates far less trig functions (only one \textit{atan2}).
\insertimageC{1}{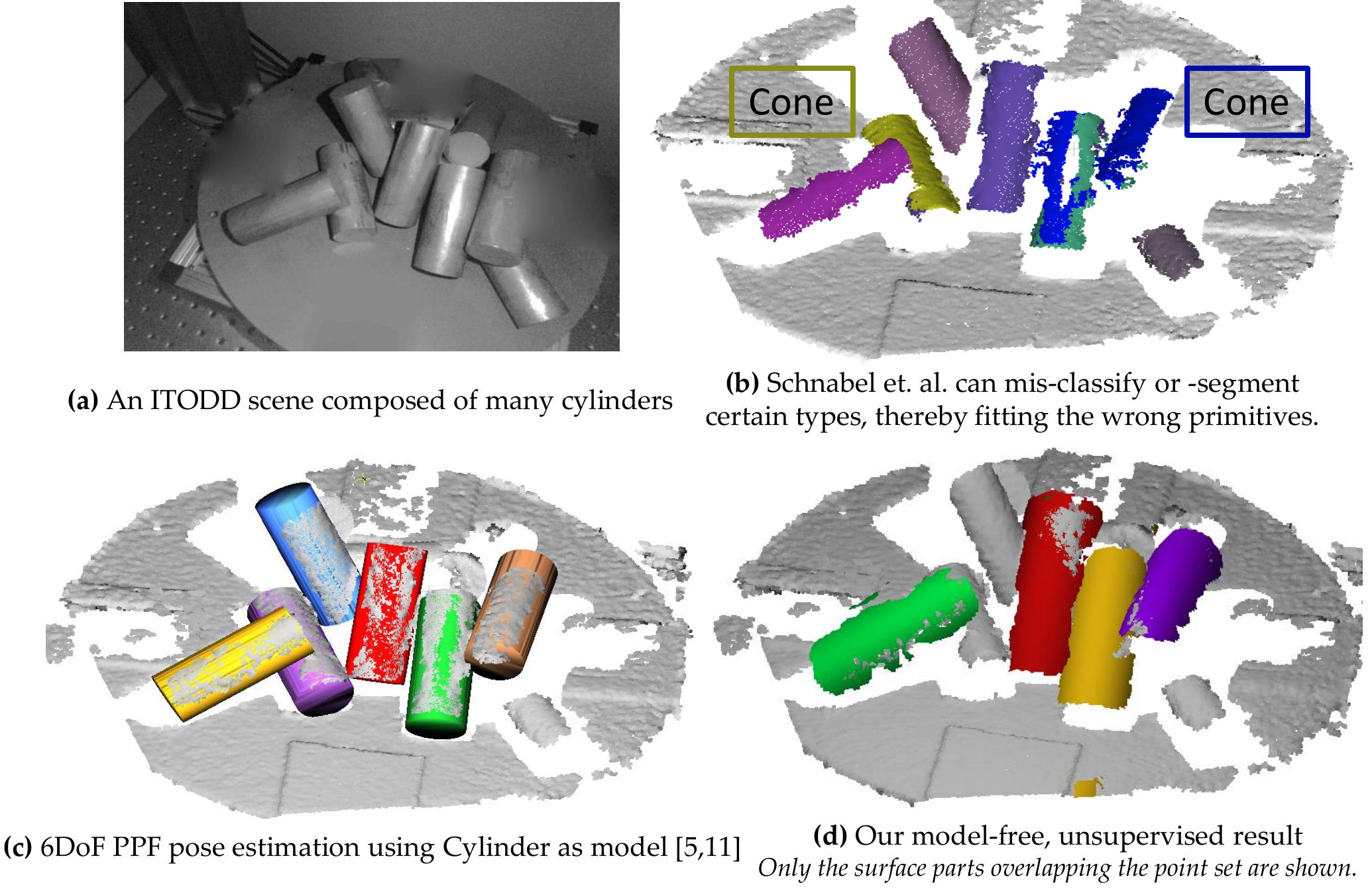}{Multiple cylinder detection in clutter and occlusions: Our approach is type agnostic and uninformed about cylinders.}{fig:cylinders}{t!}
We plug this specific fit into our detector without changing other parts and evaluate it on scenes from~\cite{Choi2016} which contains spherical everyday objects. Tab.~\ref{tab:accuracy} summarizes the dataset and reports our accuracy while Fig.~\ref{fig:spherefit} qualitatively shows that our sphere-specific detector can indeed operate in challenging real scenarios. Our algorithm is able to detect a sphere on many difficult cases, as long as the sphere is partially visible. We also do not have to specify the radius as unlike many Hough transform based methods. Note that, due to reduced basis size $(b=|\vec{b}|=1)$ this type specific fit can meet real-time criteria.
\subsubsection{Comparison to model based detectors} 
The literature is overwhelmed by the number of 3d model based pose estimation methods. Hence, we decide to compare our model-free approach to the model based ones. For that, we take the cylinders subset of the recent ITODD dataset~\cite{Drost2017} and run our generic quadric detector without training or specifying the type. Visuals of different methods are presented in Fig. \ref{fig:cylinders} whereas detection performance are reported in Tab.~\ref{tab:cylinders}. Our task is not to explicitly estimate the pose. Thus, we manually accept a hypothesis if ICP~\cite{Fitzgibbon2003} converges to a visually pleasing outcome. Note, multiple models are an important source of confusion for us, as we vote on generic quadrics. However, our algorithm outperforms certain detectors, even when we are solving a more generic problem as our shapes are allowed to deform into geometries other than cylinders. 
\begin{table}[t!]
  \centering
  \setlength\tabcolsep{2 pt}
  \resizebox{\columnwidth}{!}{
    \begin{tabular}{cccccc}
    \toprule
    PPF3D & PPF3D-E & PPF3D-E-2D & S2D~\cite{Ulrich2012}   & RANSAC~\cite{Papazov2010} & Ours \\
    \midrule
72\% & 73\% & 74\% & 24\% & 86\% & 41.9\% \\ 
    \bottomrule
    \end{tabular}%
  }
  \caption{Results on ITODD~\cite{Drost2017} cylinders: Even without looking for a cylinder, we outperfom the model based ~\cite{Ulrich2012}.}
  \label{tab:cylinders}%
\end{table}%
\insertimageC{0.9}{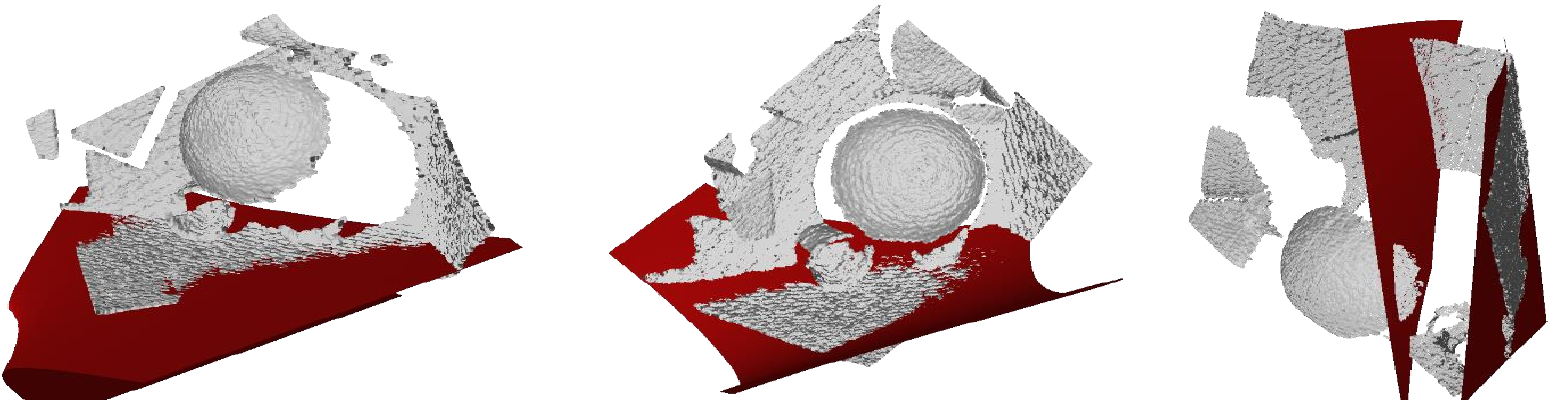}{When planes gather a majority of the votes, our method approximates them by a close non-degenerate quadric surface. This is similar to representing a large planar football field using Earth's surface.}{fig:planes}{h}
\section{Discussion and Conclusions}
We presented a fast and robust pipeline for generic primitive detection in noisy and cluttered 3D scenes. Our first contribution is a novel, linear fitting formulation for oriented point quadruplets. We thoroughly analyzed this fit and devised an efficient null-space voting which uses three pieces of point primitives plus a simple local search instead of a full four oriented point fit. Together, the fitting and voting, establish the minimalist cases known up to now - three oriented points, potentially paving the way towards real-time operation. While our detector targets a generic surface, we can, optionally, convert to a type-specific fit to boost speed and accuracy. 

Unless made specific, our method is surpassed by type-specific fits in detection rate since solving the generic problem is more difficult. It remains an open issue to bring the performance to the levels of type-specific fits.
Nevertheless, if the design matrix $\bm{A}$ targets a specific type, we perform even better. Degenerate cases are also difficult for us as shown in Fig. \ref{fig:planes}, but we always find a non-degenerate configuration good-enough to approximate the primitive.

\vspace{-3mm}
\ifCLASSOPTIONcompsoc
  \section*{Acknowledgments}
\else
  \section*{Acknowledgment}
\fi

The authors would like to thank Bertram Drost and Maximilian Baust for fruitful discussions.

\ifCLASSOPTIONcaptionsoff
  \newpage
\fi

{\small
\bibliographystyle{IEEEtran}
\bibliography{egbib}
}



%

%
\begin{IEEEbiography}[{\includegraphics[width=1.1in,height=1.25in,keepaspectratio]{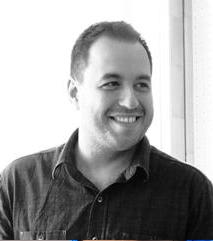}}]{Tolga Birdal} is a PhD candidate at the Chair for Computer Aided Medical Procedures / TU Munich and a Doktorand at Siemens AG. He completed his Bachelors as an Electronics Engineer at Sabanci University in 2008. Subsequently, he studied Computational Science and Engineering at TU Munich. In continuation to his Master's thesis on \"3D Deformable Surface Recovery Using RGBD Cameras\", he now focuses his research and development on large object detection, pose estimation and reconstruction. Recently, he is awarded Ernst von Siemens Scholarship 2015 and received the EMVA Young Professional Award 2016 for his PhD work.
\end{IEEEbiography}



\begin{IEEEbiography}[{\includegraphics[width=1.1in,height=1.35in,keepaspectratio]{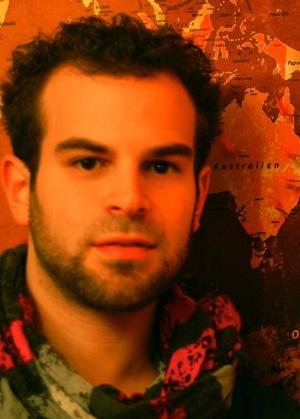}}]{Benjamin Busam} is a PhD candidate with the Chair for Computer Aided Medical Procedures, Technical University of Munich and Head of Research at FRAMOS Imaging Systems. He finished his Bachelors in Mathematics at TUM in 2011. In his subsequent postgraduate programme, he studied Mathematics and Physics at ParisTech, France and at University of Melbourne, Australia, before he graduated with distinction at TU Munich in 2014. In continuation to his Master's thesis on “Projective Geometry and 3D point cloud matching”, for which he received the EMVA Young Professional Award 2015, he now focuses his research and development on 2D/3D object detection and pose estimation for both outside-in and inside-out tracking.
\end{IEEEbiography}

\begin{IEEEbiography}[{\includegraphics[width=1.1in,height=1.3in,keepaspectratio]{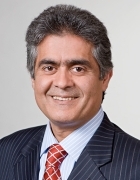}}]{Nassir Navab} is a professor of computer science and founding director of Computer Aided Medical Procedures and Augmented Reality (CAMP) laboratories at Technische Universit{\"a}t M{\"u}nchen (TUM) and Johns Hopkins University (JHU). He received the Ph.D. degree from INRIA and University of Paris XI, France, and enjoyed two years of postdoctoral fellowship at MIT Media Laboratory before joining Siemens Corporate Research (SCR) in 1994. At SCR, he was a distinguished member and received the Siemens Inventor of the Year Award in 2001. In 2012, he was elected as a fellow member of MICCAI society. He received the 10 year lasting Impact Award of IEEE ISMAR in 2015. He holds 45 granted US patents and over 40 European ones and served on the program and organizational committee of over 80 international conferences including CVPR, ICCV, ECCV, MICCAI, and ISMAR. His current research interests include robotic imaging, computer aided surgery, computer vision and augmented reality.
\end{IEEEbiography}

\begin{IEEEbiography}[{\includegraphics[width=1.1in,height=1.25in,keepaspectratio]{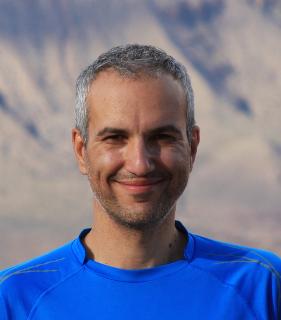}}]{Slobodan Ilic} is currently senior key expert research scientist at Siemens Corporate Technology in Munich, Perlach. He is also a visiting researcher and lecturer at Computer Science Department of TUM and closely works with the CAMP Chair. From 2009 until end of 2013 he was leading the Computer Vision Group of CAMP at TUM, and before that he was a senior researcher at Deutsche Telekom Laboratories in Berlin. In 2005 he obtained his PhD at EPFL in Switzerland under supervision of Pascal Fua. His research interests include: 3D reconstruction, deformable surface modelling and tracking, real-time object detection and tracking, human pose estimation and semantic segmentation.
\end{IEEEbiography}

\begin{IEEEbiography}[{\includegraphics[width=1.1in,height=1.25in,keepaspectratio]{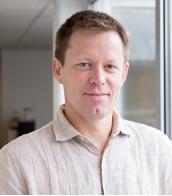}}]{Peter Sturm} obtained MSc degrees from INPG (National Polytechnical Institute of Grenoble, France) and the University of Karlsruhe, both in 1994, and a Ph.D. degree from INPG in 1997, with Long Quan as advisor. His Ph.D. thesis received the SPECIF award. After a two-year postdoc at Reading University, working with Steve Maybank, he joined Inria on a permanent research position as Chargé de Recherche in 1999. Since 2006, he is Directeur de Recherche (the INRIA equivalent of Professor). Since 2015, he is Deputy Scientific Director of Inria. His main research topics have been in Computer Vision, and specifically related to camera (self-)calibration, 3D reconstruction and motion estimation. In 2011, Peter joined the STEEP research team, which is working toward contributing to sustainable development and on the use of integrated land use and transportation models for urban areas, in particular.
\end{IEEEbiography}





\end{document}